\documentclass{article}

\usepackage{hyperref}

\usepackage[accepted]{icml2024}

\usepackage{amsmath}
\usepackage{amssymb}
\usepackage{mathtools}
\usepackage{amsthm}

\usepackage[capitalize,noabbrev]{cleveref}

\usepackage{thm-restate}  % Must come before theorem declarations or refs will be wrong
\theoremstyle{plain}
\newtheorem{theorem}{Theorem}[section]

\newtheorem{lemma}[theorem]{Lemma}

\theoremstyle{definition}

\newtheorem{assumption}[theorem]{Assumption}
\theoremstyle{remark}

\usepackage[textsize=tiny]{todonotes}

\icmltitlerunning{Averaging $n$-step Returns Reduces Variance}

\usepackage{paper_style}
\graphicspath{{images/}}

\usepackage{amsmath,amsfonts,bm}

\def\eqref#1{equation~\ref{#1}}
\def\1{\bm{1}}

\def\vone{{\bm{1}}}

\def\vc{{\bm{c}}}
\def\vd{{\bm{d}}}

\def\mD{{\bm{D}}}

\def\mP{{\bm{P}}}

\def\mPhi{{\bm{\Phi}}}
\def\mPi{{\bm{\Pi}}}

\DeclareMathAlphabet{\mathsfit}{\encodingdefault}{\sfdefault}{m}{sl}
\SetMathAlphabet{\mathsfit}{bold}{\encodingdefault}{\sfdefault}{bx}{n}

\newcommand{\E}{\mathbb{E}}

\newcommand{\R}{\mathbb{R}}

\newcommand{\Var}{\mathrm{Var}}

\newcommand{\Cov}{\mathrm{Cov}}
\begin{document}

\twocolumn[
\icmltitle{Averaging $n$-step Returns Reduces Variance in Reinforcement Learning}



\icmlsetsymbol{equal}{*}

\begin{icmlauthorlist}
\icmlauthor{Brett Daley}{uofa,amii}
\icmlauthor{Martha White}{uofa,amii,cifar}
\icmlauthor{Marlos C.~Machado}{uofa,amii,cifar}
\end{icmlauthorlist}

\icmlaffiliation{uofa}{Department of Computing Science, University of Alberta, Edmonton, AB, Canada}
\icmlaffiliation{amii}{Alberta Machine Intelligence Institute}
\icmlaffiliation{cifar}{Canada CIFAR AI Chair}

\icmlcorrespondingauthor{Brett Daley}{brett.daley@ualberta.ca}

\icmlkeywords{Reinforcement Learning, Temporal-Difference Learning, Credit Assignment, Variance Reduction, Sample Efficiency}

\vskip 0.3in
]



\printAffiliationsAndNotice{}  % leave blank if no need to mention equal contribution

\begin{abstract}
    Multistep returns, such as $n$-step returns and $\lambda$-returns, are commonly used to improve the sample efficiency of reinforcement learning (RL) methods.
    The variance of the multistep returns becomes the limiting factor in their length;
    looking too far into the future increases variance and reverses the benefits of multistep learning.
    In our work, we demonstrate the ability of compound returns---weighted averages of $n$-step returns---to reduce variance.
    We prove for the first time that any compound return with the same contraction modulus as a given $n$-step return has strictly lower variance.
    We additionally prove that this variance-reduction property improves the finite-sample complexity of temporal-difference learning under linear function approximation.
    Because general compound returns can be expensive to implement, we introduce two-bootstrap returns which reduce variance while remaining efficient, even when using minibatched experience replay.
    We conduct experiments showing that compound returns often increase the sample efficiency of $n$-step deep RL agents like DQN and PPO.
    \looseness=-1
\end{abstract}

\section{Introduction}

Efficiently learning value functions is critical for reinforcement learning (RL) algorithms.
Value-based RL methods \citep[e.g.,][]{watkins1989learning,rummery1994line,mnih2015human} encode policies implicitly in a value function, and policy evaluation is the principal mechanism of learning.
Even when RL methods instead learn parametric policies, accurate value functions are needed to guide policy improvement \citep[e.g.,][]{silver2014deterministic,lillicrap2016continuous,fujimoto2018addressing,haarnoja2018soft} or to serve as a baseline
\citep[e.g.,][]{barto1983neuronlike,sutton1984temporal,williams1992simple,schulman2015trust}.
Quicker policy evaluation is therefore useful to many RL algorithms.

One way to achieve faster value-function learning is through multistep returns, in which more than one reward following an action is used to construct a prediction target.
Multistep returns are used extensively in deep RL
\citep[e.g.,][]{schulman2015high,mnih2016asynchronous,munos2016safe,schulman2017proximal,hessel2018rainbow,schrittwieser2020mastering,wurman2022outracing,chebotar2023q,schwarzer2023bigger},
where the value function is approximated by a neural network.
In theory, multistep returns incorporate more information regarding future outcomes, leading to faster credit assignment and, in turn, faster learning.
However, faster learning is not guaranteed in practice because looking farther into the future increases variance and can end up requiring more samples to estimate the mean.
These two opposing factors must be balanced to achieve fast and stable learning.
The most common multistep returns are $n$-step returns and $\lambda$-returns,
both of which span between standard temporal-difference \citep[TD;][]{sutton1988learning} learning ($n=1$ or $\lambda=0$) and Monte Carlo learning ($n \to \infty$ or $\lambda = 1$).
\looseness=-1

Implementation has generally been the main consideration in choosing which of these return estimators to use.
When the value function is a lookup table or a linear parametric function, the $\lambda$-return is preferred for its efficient implementation using TD($\lambda$) with eligibility traces \citep{sutton1988learning}.
However, in off-policy deep RL, the value function is a neural network trained with an experience replay buffer \citep{lin1992self}, so the extra bootstrapping performed by the $\lambda$-return becomes costly and $n$-step returns are more common.
Although recent work has explored ways to efficiently train neural networks using replayed $\lambda$-returns \citep[e.g.,][]{munos2016safe,harb2016investigating,daley2019reconciling}, $\lambda$-returns generally remain more expensive or complex to implement than $n$-step returns.

Despite its more complicated implementation in deep RL, the $\lambda$-return is interesting because it averages many $n$-step returns together.
In this paper, we show that this averaging reduces variance compared to $n$-step returns and that the variance reduction leads to faster learning.
More generally, the $\lambda$-return is an example of a \emph{compound} return, or a weighted average of two or more $n$-step returns \citep[][Sec.~12.1]{sutton2018reinforcement}.
The $\lambda$-return is the most prevalent compound return, but others such as the $\gamma$-return \citep{konidaris2011td_gamma} and the $\Omega$-return \citep{thomas2015policy} have also been shown to be effective.
While compound returns are known to be theoretically sound \citep[][Sec.~7.2]{watkins1989learning}, they have not been extensively analyzed.
Compound returns are ideally suited for deep RL because experiences are already stored in a replay memory, making it easy to compute several $n$-step returns cheaply and then average them together in any desired way.

In this paper, we formally investigate the variance of compound returns, deriving a variance model for arbitrary compound returns.
Additionally, we prove that compound returns have a \emph{variance-reduction property}:
any compound return with the same contraction modulus as an $n$-step return has strictly lower variance if the TD errors are not perfectly correlated.
As a corollary, there exists a $\lambda$-return for every $n$-step return with the same contraction modulus but lower variance, implying a better bias-variance trade-off.
\looseness=-1

Our theoretical results suggest one should prefer $\lambda$-returns over $n$-step returns, but these are expensive to implement for deep RL.
Thus, we also introduce an efficient approximation called \textbf{Pi}ecewise $\bm{\lambda}$-\textbf{R}eturn (PiLaR or Pilar).
Pilar is computed by averaging just two $n$-step returns together---the most efficient compound return possible.
The lengths of the $n$-step returns are chosen to put weights on the TD errors that are close to those assigned by TD($\lambda$), thereby achieving similar variance reduction as the $\lambda$-return.
We show that Pilars improve sample efficiency compared to $n$-step returns when used to train Deep Q-Network \citep[DQN;][]{mnih2015human}, which is the foundation of many other off-policy deep RL methods.
In on-policy methods like PPO \citep{schulman2017proximal} where such an approximation is not necessary, we also show similar improvements for $\lambda$-returns, also known as Generalized Advantage Estimation in this context \citep{schulman2015high}, over $n$-step returns.
\looseness=-1

\section{Background}

Value-based RL agents interact with their environments to iteratively improve estimates of their expected cumulative reward.
By convention, the agent-environment interaction is modeled as a Markov decision process (MDP) described by the tuple $(\mathcal{S},\mathcal{A},p,r)$.
At time $t$, the agent observes the environment state, $S_t \in \mathcal{S}$, and selects an action, $A_t \in \mathcal{A}$, accordingly.
The environment then transitions to a new state, $S_{t+1} \in \mathcal{S}$, with probability $p(S_{t+1} \mid S_t,A_t)$, and returns a scalar reward, $\smash{R_{t+1} \defeq r(S_t,A_t,S_{t+1})}$.
We assume the agent samples each action with probability $\pi(A_t \mid S_t)$, where $\pi$ is a stochastic policy.

In the standard prediction problem, the agent's goal is to learn a value function
$v \colon \mathcal{S} \to \mathbb{R}$
that estimates the expected discounted return $v_\pi(s)$ attainable in each state~$s$.
Letting
$\smash{G_t \defeq \sum_{i=0}^\infty \gamma^i R_{t+1+i}}$
be the observed Monte Carlo return at time $t$, where $\gamma \in [0,1]$ is a discount factor and actions are implicitly sampled from $\pi$, then
$\smash{v_\pi(s) \defeq \mathbb{E} \! \left[G_t \mid S_t = s \right]}$.

The basic learning operation of value-based RL is a \emph{backup}, which has the general form
\begin{equation}
    \label{eq:backup}
    v(S_t) \gets v(S_t) + \alpha_t \left( \hat{G}_t - v(S_t) \right)
    \,,
\end{equation}
where $\smash{\hat{G}}_t$ is a return estimate (the \emph{target}) and $\alpha_t \in (0,1]$ is the step size.
Substituting various estimators for $\smash{\hat{G}}_t$ leads to different learning properties.
For instance, the Monte Carlo return could be used, but it suffers from high variance, and delays learning until the end of an episode.
To reduce the variance and delay, return estimates can \emph{bootstrap} from the value function.
Bootstrapping is the fundamental mechanism underlying TD learning \citep{sutton1988learning}, and the extent to which the chosen estimator, $\hat{G}_t$, bootstraps is a major factor in learning performance.

The most basic multistep version of TD learning uses the $n$-step return as its target in \Cref{eq:backup}:
\begin{equation}
    \label{eq:nstep_return}
    \nstep{n}_t
    \defeq \sum_{i=0}^{n-1} \gamma^i R_{t+1+i} + \gamma^n v(S_{t+n})
    \,.
\end{equation}
Bootstrapping introduces bias in the update, since
generally
$\mathbb{E} \left[v(S_{t+n}) \mid S_t \right] \neq v_\pi(S_{t+n})$
before convergence, but it greatly reduces variance.
The case of $n=1$ corresponds to the classic TD(0) algorithm,
$v(S_t) \gets v(S_t) + \alpha_t \delta_t$,
where
$\smash{\delta_t \defeq R_{t+1} + \gamma v(S_{t+1}) - v(S_t)}$
is the TD error.
However, bootstrapping after just one time step is slow because long-term reward information must propagate indirectly through the value function, requiring many behavioral repetitions before $v$ approximates $v_\pi$ well.
Larger values of $n$ consider more rewards per update and assign credit faster, but at the price of increased variance, with $n \to \infty$ reverting to the Monte Carlo return.
The choice of $n$ thus faces a bias-variance trade-off \citep{kearns2000bias}, with intermediate values typically performing best in practice \citep[][Sec.~7.1]{sutton2018reinforcement}.

Another type of multistep return is the $\lambda$-return, used by TD($\lambda$) algorithms \citep{sutton1988learning}, which is equivalent to an exponentially weighted average of $n$-step returns:
\halfglue
\begin{equation*}
    G^\lambda_t \defeq (1-\lambda) \sum_{n=1}^\infty \lambda^{n-1} \nstep{n}_t
    \,,
\end{equation*}
where $\lambda \in [0,1]$.
The $\lambda$-return is just one way to average $n$-step returns together, but any weighted average is possible.
Such averages are known as compound returns \citep[][Sec.~12.1]{sutton2018reinforcement}, formally expressed as
\halfglue
\begin{equation}
    \label{eq:compound_return}
    G^\vc_t \defeq \sum_{n=1}^\infty c_n \nstep{n}_t
    \,,
\end{equation}
where $\vc$ is an infinite sequence of nonnegative weights with the constraint $\sum_{n=1}^\infty c_n = 1$.
\Cref{eq:compound_return} is a strict generalization of all of the aforementioned return estimators;
however, it technically constitutes a \emph{compound} return if and only if at least two of the weights are nonzero---otherwise, it reduces to an $n$-step return.
All choices of weights that sum to $1$ are equally valid in the sense that using their compound returns as the target in \Cref{eq:backup} will converge to $v_\pi$ under the standard step-size criteria of \citet{robbins1951stochastic}.
However, in practice, the choice of weights greatly affects the bias and variance of the estimator and, as a consequence, the empirical rate of convergence to $v_\pi$.
Furthermore, different choices of weights will vary in the amount of computation needed;
sparse weights require less bootstrapping and therefore less computation.
A principal contribution of this paper is to shed light on the trade-offs between these factors and to develop new compound returns that favorably balance these considerations (see \Cref{sec:approximate}).

Many RL agents are control agents, meaning they do not just predict $v_\pi$ for a fixed policy, but also use these estimates to update the behavior policy during training.
One way to do this is Q-learning \citep{watkins1989learning}.
Rather than learning a state-value function $v$, the agent learns an action-value function
$q \colon \mathcal{S} \times \mathcal{A} \to \mathbb{R}$
that estimates the expected return $q_*(s,a)$ earned by following an optimal policy after taking action $a$ in state $s$.
The estimated value of a state is therefore $v(s) = \max_{a \in \mathcal{A}} q(s,a)$, and all of the previously discussed return estimators apply after making this substitution.
Backups are conducted as before, but now operate on $q(S_t,A_t)$ instead of $v(S_t)$.

Learning is off-policy in this setting, since the agent now predicts returns for a greedy policy that differs from the agent's behavior.
Any multistep return is therefore biased unless the expectation is explicitly corrected, e.g., by importance sampling \citep{kahn1951estimation}.
However, it is common practice to ignore this bias in deep RL, and recent research has even suggested that doing so is more effective with both $n$-step returns \citep{hernandez2019understanding} and $\lambda$-returns \citep{daley2019reconciling,kozuno2021revisiting}, the latter of which become Peng's Q($\lambda$) \citep{peng1996incremental}.
For these reasons, we also forgo off-policy corrections in our work.

\section{Variance Analysis}
\label{sec:analysis}

Our main goal in this section is to derive conditions for when a compound return reduces variance compared to a given $n$-step return.
We call this the \emph{variance-reduction property} (see \Cref{theorem:variance_reduction}).
An important consequence of this property is that when these conditions are met and both returns are chosen to have the same contraction modulus, the compound return needs fewer samples than the $n$-step return to converge (see \Cref{theorem:ft_analysis}).

The outline of this section is as follows.
We first derive variance models for $n$-step returns and compound returns under generalized assumptions extending prior work (\Cref{subsec:characterizing_variance}) and then develop a way to pair different return families by equating their worst-case bias (\Cref{subsec:erp}).
This sets up our principal theoretical result, the variance-reduction property of compound returns (\Cref{subsec:vrp}).
We also conduct a finite-time analysis of TD learning with compound returns to show that this variance reduction does improve theoretical sample complexity (\Cref{subsec:ft_analysis}).
Finally, we experimentally demonstrate that $\lambda$-returns---an instance of compound returns---outperform $n$-step returns in a prediction task as indicated by our theory (\Cref{subsec:lambda_case_study}).

\subsection{Characterizing Compound-Return Variance}
\label{subsec:characterizing_variance}

\halfglue

We start by developing reasonable models for the variance of $n$-step returns and compound returns, building on the work of \citet{konidaris2011td_gamma}.
We note that our real quantity of interest in this section is the conditional variance of the backup error
$\smash{\hat{G}_t - v(S_t)}$.
The degree of this quantity's deviation from its expected value is what ultimately impacts the performance of value-based RL methods in \Cref{eq:backup}.
Nevertheless, this turns out to be the same as the conditional variance of the return $\smash{\hat{G}_t}$, since $v(S_t)$ contributes no randomness:
$\smash{\Var[\hat{G}_t - v(S_t) \mid S_t]} = \smash{\Var[\hat{G}_t \mid S_t]}$.
This equivalence allows us to interchange the variances of a return and its error, depending on which is more computationally convenient.

Modeling a compound return's variance is challenging because it typically requires making assumptions about how the variance of $n$-step returns increase as a function of $n$, as well as how strongly correlated different $n$-step returns are.
If these assumptions are too strong, the derived variances fail to reflect reality and lead to poorly informed algorithmic choices.
For instance, consider the following compound return, which we call a \emph{two-bootstrap return}:
\halfglue
\begin{equation}
    \label{eq:two-bootstrap}
    (1-c) \nstep{n_1}_t + c \nstep{n_2}_t
    \!,\,
    c \in (0,1)
    ,\,
    n_1 < n_2
    \,.
\end{equation}
Let $\smash{\sigma_n^2 \defeq \Var[\nstep{n}_t \mid S_t]}$.
The variance of the above is
\halfglue
\begin{equation*}
    (1-c)^2 \sigma_1^2 + c^2 \sigma_2^2 + 2c(1-c)\sigma_1 \sigma_2 \rho
    \,,
\end{equation*}
where $\rho = \Corr[\nstep{n_1}_t,\nstep{n_2}_t \mid S_t]$.
To evaluate this expression, it is tempting to assume either $\rho = 0$ to remove the covariance term, or $\rho = 1$
because both returns are generated from the same trajectory.
However, neither assumption is fully correct.
We can see this more clearly by decomposing the returns into their constituent TD errors:
\begin{align*}
    \nstep{n_1}_t &= v(S_t) + \sum_{i=0}^{n_1-1} \gamma^i \delta_{t+i}
    \,,\\*  % No page break here
    \nstep{n_2}_t &= v(S_t) + \sum_{i=0}^{n_1-1} \gamma^i \delta_{t+i} + \sum_{i=n_1}^{n_2-1} \gamma^i \delta_{t+i}
    \,.
\end{align*}
Because the returns share the first $n_1$ TD errors,
averaging them as in \Cref{eq:two-bootstrap} can only reduce the variance from the last $n_2-n_1$ TD errors.
Two different $n$-step returns are therefore neither uncorrelated nor perfectly correlated;
they consist of various random elements, some of which are shared and cannot be averaged, and others which are not shared and can be averaged.
Models that fail to account for this partial averaging lead to poor variance predictions.

To be accurate, our variance analysis should therefore start with the TD error as the fundamental unit of randomness within a return.
In the following analysis, we generalize the $n$-step return variance assumptions made by \citet[][Sec.~3]{konidaris2011td_gamma} to obtain an expression for the covariance between two TD errors.
We assume that there exist $\kappa \geq 0$ and $\rho \in [0,1]$ such that the following hold:
\begin{restatable}{assumption}{tdvar}
    \label{assump:td_var}
    Each TD error has uniform variance:
    $\Var[\delta_{t+i} \mid S_t] = \kappa$,\enskip $\forall~i \geq 0$.
\end{restatable}
\begin{restatable}{assumption}{tdcovar}
    \label{assump:td_cov}
    TD errors are uniformly correlated:
    $\Corr[\delta_{t+i},\delta_{t+j} \mid S_t] = \rho$,\enskip $\forall~i \geq 0$, $\forall~j \neq i$.
\end{restatable}
We must make some assumptions about the TD errors because not much can be said otherwise about the variance of a return estimate for an arbitrary MDP.
For example, it is possible to contrive problems in which $1$-step TD learning actually has much higher variance than Monte Carlo estimation.
Consider an MDP with reward zero everywhere and a value function initialized with random numbers;
the Monte Carlo return would not have variance but TD bootstrapping would.
However, it is generally \emph{expected} that $1$-step TD has much lower variance than Monte Carlo, because summing more TD errors (i.e., random variables) together will very likely increase practical variance.
This intuition motivates our assumptions, as we want a variance model that captures this behavior while remaining tractable for analysis.
We further discuss the justification and historical context of these assumptions in \Cref{app:assumptions}.

Now, we can unify \Cref{assump:td_var,assump:td_cov} as
\begin{equation}
    \label{eq:assumps_combined}
    \Cov[\delta_{t+i},\delta_{t+j} \mid S_t]
    = ((1-\rho) \vone_{i=j} + \rho) \kappa
    \,,
\end{equation}
where $\vone_{i=j} = 1$ if $i=j$, and is $0$ otherwise.
In the proposition below, we derive a variance model for the $n$-step return by decomposing the return into a sum of discounted TD errors and then adding up the pairwise covariances given by \Cref{eq:assumps_combined}.
For brevity, we define a function $\gammafunc{k}{n}$ to represent partial sums of a geometric series with common ratio~$\gamma^k$.
That is,
$\gammafunc{k}{n} = (1 - \gamma^{kn}) \mathbin{/} (1 - \gamma^k)$ for $\gamma < 1$\,,
and $\gammafunc{k}{n} = n$ for $\gamma = 1$\,.

\begin{restatable}{proposition}{propnstepvariance}
    \label{prop:nstep_variance}
    The variance of an $n$-step return is
    \begin{equation*}
        \Var[\nstep{n}_t \mid S_t]
        = (1-\rho) \gammafunc{2}{n} \kappa + \rho \gammafunc{1}{n}^2 \kappa
        \,.
    \end{equation*}
\end{restatable}
\glue
\begin{proof}
    See \Cref{app:prop_nstep_variance}.
\end{proof}
\glue

Our $n$-step variance model linearly interpolates between an optimistic case where TD errors are uncorrelated ($\rho=0$) and a pessimistic case where TD errors are maximally correlated ($\rho=1$).
In the maximum-variance scenario of $\gamma = 1$, we have
$\gammafunc{1}{n} = \gammafunc{2}{n} = n$,
so the model becomes
$(1-\rho) n \kappa + \rho n^2 \kappa$\,, i.e., it interpolates between linear and quadratic functions.
In \Cref{app:how_real_assumptions}, we demonstrate that these bounds are consistent with empirical data in real environments, even when our assumptions do not fully hold.

\Cref{eq:assumps_combined} enables us to go beyond $n$-step returns and calculate variances for arbitrary compound returns.
We accomplish this by again decomposing the return into a weighted sum of TD errors (see the next lemma) and then applying our assumptions to derive a compound variance model in the following proposition.

\begin{restatable}{lemma}{lemmacompounderror}
    \label{lemma:compound_error}
    A compound error can be written as a weighted summation of TD errors:
    \begin{equation*}
        \label{eq:compound_error}
        G^\vc_t - v(S_t)
        = \sum_{i=0}^\infty \gamma^i h_i \delta_{t+i}
        \,,
        \text{ where } h_i \defeq \smashoperator{\sum_{n=i+1}^\infty} c_n
        \,.
    \end{equation*}
\end{restatable}
\begin{restatable}{proposition}{propcompoundvariance}
    \label{prop:compound_variance}
    The variance of a compound return is
    \begin{equation*}
        \label{eq:compound_variance}
        \Var[G^\vc_t \mid S_t]
        = (1-\rho) \sum_{i=0}^\infty \gamma^{2i} h_i^2 \kappa + \rho \sum_{i=0}^\infty \sum_{j=0}^\infty \gamma^{i+j} h_i h_j \kappa
        \,.
    \end{equation*}
\end{restatable}
\glue
\begin{proof}[Proofs]
    See \Cref{app:lemma_compound_error,app:prop_compound_variance}.
\end{proof}
\glue

The cumulative weights $(h_i)_{i=0}^\infty$ fully specify the variance of a compound return.
For instance, the $\lambda$-return assigns a cumulative weight of
$h_i = {\sum_{n=i+1}^\infty (1-\lambda) \lambda^{n-1} = \lambda^i}$
to the TD error $\delta_{t+i}$, which matches the TD($\lambda$) algorithm \citep{sutton1988learning}.
Substituting this weight into \Cref{eq:compound_variance} and solving the geometric series yields the variance for the $\lambda$-return, which we show in \Cref{app:lambda_variance}.

\subsection{Error-Reduction Property and Effective $n$-step}
\label{subsec:erp}

\halfglue

\Cref{prop:compound_variance} provides a method for calculating variance, but we would still like to show that compound returns reduce variance relative to $n$-step returns.
To do this, we first need a way to relate two returns in terms of their expected performance.
This is because low variance by itself is not sufficient for fast learning;
for example, the $1$-step return has very low variance, but learns slowly.

In the discounted setting, a good candidate for expected learning speed is the return's contraction modulus.
The contraction modulus is the constant factor by which the maximum value-function error between $v$ and $v_\pi$ is guaranteed to be reduced.
When the contraction modulus is less than~$1$, the return estimator exhibits an \emph{error-reduction property}:
i.e., the maximum error decreases on average with every backup iteration of \Cref{eq:backup}.
This property is commonly used in conjunction with the Banach fixed-point theorem to prove that $v$ eventually converges to $v_\pi$ \citep[see, e.g.,][Sec.~4.3]{bertsekas1996neuro}.
The error-reduction property of compound returns was first identified by \citet[][Sec.~7.2]{watkins1989learning} and is expressed formally as
\begin{align}
    \nonumber
    &\max_{s \in \mathcal{S}} \abs{\mathbb{E}\left[ G^\vc_t \mid S_t = s \right] - v(s)} \\*
    \label{eq:compound_erp}
    &\qquad\leq \left(\sum_{k=1}^\infty c_k \gamma^k \right) \max_{s \in \mathcal{S}} \abs{v(s) - v_\pi(s)}
    \,.
\end{align}
The contraction modulus is the coefficient on the right-hand side:
$\sum_{k=1}^\infty c_k \gamma^k$\,,
a weighted average of each individual $n$-step return's contraction modulus, $\gamma^n$.
We can compare compound returns to $n$-step returns that have equivalent error-reduction properties by solving the equation
${\gamma^n = \sum_{k=1}^\infty c_k \gamma^k}$
for $n$.
We call this the \emph{effective $n$-step} of the compound return, since the compound return reduces the worst-case error as though it were an $n$-step return whose length is the solution to the previous equation (see \Cref{prop:effective_nstep} below).
In undiscounted settings, we cannot directly equate contraction moduli like this because they all become
${\sum_{k=1}^\infty c_k = 1}$,
but we can still solve the limit as $\gamma \to 1$ to define the effective $n$-step.

\begin{restatable}[Effective $n$-step of compound return]{proposition}{propeffectivenstep}
    \label{prop:effective_nstep}
    Let $G^\vc_t$ be any compound return and let
    \begin{equation}
        \label{eq:effective_nstep}
        n = \begin{cases}
            \log_\gamma \left(\sum_{k=1}^\infty c_k \gamma^k\right) \,, & \text{if } 0 < \gamma < 1 \,, \\
            \sum_{k=1}^\infty c_k k \,, & \text{if } \gamma = 1 \,.
        \end{cases}
    \end{equation}
    When $n$ is an integer, $G^\vc_t$ shares the same bound in \Cref{eq:compound_erp} as the $n$-step return $\smash{\nstep{n}_t}$.
\end{restatable}
\glue
\begin{proof}
    See \Cref{app:prop_effective_nstep}.
\end{proof}
\glue

With or without discounting, an $n$-step return and a compound return that satisfy \Cref{eq:effective_nstep} have the same error-reduction property.
We refer to the quantity $\sum_{k=1}^\infty c_k k$ as the \emph{center of mass} (COM) of a return, since it is the first moment of the weight distribution over $n$-step returns.
Intuitively, this represents the average length into the future considered by the return, and is the undiscounted analog of the log contraction modulus.

\subsection{The Variance-Reduction Property}
\label{subsec:vrp}

\halfglue

With the previous definitions, we are now ready to formalize the \emph{variance-reduction property} of compound returns.

\begin{restatable}[Variance-reduction property of compound returns]{theorem}{theoremvariancereduction}
    \label{theorem:variance_reduction}
    Let $\smash{\nstep{n}_t}$ be any $n$-step return and let $\smash{G^\vc_t}$ be any compound return with the same effective $n$-step:
    i.e., $\vc$ satisfies \Cref{prop:effective_nstep}.
    The inequality
    $\smash{\Var[G^\vc_t \mid S_t] \leq \Var[\nstep{n}_t \mid S_t]}$
    always holds, and is strict whenever TD errors are not perfectly correlated ($\rho < 1$).
\end{restatable}
\glue
\begin{proof}
    See \Cref{app:theorem_variance_reduction}.
\end{proof}
\glue

\Cref{theorem:variance_reduction} shows that whenever a compound return has the same contraction modulus ($\gamma < 1$) or COM ($\gamma = 1$) as an $n$-step return, it has lower variance as long as the TD errors are not perfectly correlated.
Perfect correlation between all TD errors would be unlikely to occur except in contrived, maximum-variance MDPs;
thus, compound returns reduce variance in most cases.
Crucially, variance reduction is achieved for any type of weighted average---although the magnitude of reduction does depend on the specific choice of weights.
The exact amount, in terms of $\kappa$, can be calculated by subtracting the compound variance from the $n$-step variance for a given contraction modulus or COM.
As an example, we bound the variance reduction of a $\lambda$-return in the following corollary.

\begin{restatable}[Variance reduction of $\lambda$-return]{corollary}{corolambdavariancereduction}
    \label{coro:lambda_variance_reduction}
    The magnitude of variance reduction for a $\lambda$-return is bounded by
    \begin{equation*}
        0
        \leq \Var[G^n_t \mid S_t] - \Var[G^\lambda_t \mid S_t]
        \leq (1-\rho) \frac{\lambda}{1-\lambda^2} \kappa
        \,.
    \end{equation*}
    This magnitude is monotonic in $\gamma$ and maximized at $\gamma = 1$.
\end{restatable}
\glue
\begin{proof}
    See \Cref{app:lambda_variance_reduction}.
\end{proof}
\glue

The magnitude of this variance reduction also increases monotonically as $\lambda \to 1$, showing that the potential for variance reduction improves as the effective $n$-step of the return increases.
Interestingly, this is reflected in our later random-walk experiments in \Cref{subsec:lambda_case_study} (observe the performance gaps in \Cref{fig:rw19} when $\alpha \approx 1$).
Future work can derive the weights that maximize the variance reduction in \Cref{theorem:variance_reduction}, but this will likely require more advanced techniques such as functional analysis and is beyond the scope of this paper.

\subsection{Finite-Time Analysis}
\label{subsec:ft_analysis}

\halfglue

It might not be obvious that reducing variance leads to faster learning;
indeed, in other settings such as direct policy optimization, this is not always the case \citep[see][]{chung2021beyond}.
We conduct a finite-time analysis of compound TD learning to prove that lower variance does lead to faster learning in this setting.
We consider linear function approximation, where
$v_\theta(s) = \phi(s)\tran \theta$
for features $\phi(s) \in \R^d$ and parameters $\theta \in \R^d$;
note that tabular methods can be recovered using one-hot features.
The parameters are iteratively updated according to
\begin{align}
    \label{eq:td_lfa}
    &\theta_{t+1}
    = \theta_t + \alpha g^\vc_t(\theta_t)
    \,,\\*  % No page break here
    \nonumber
    &\qquad \text{where }
    g^\vc_t(\theta) \defeq \left(G^\vc_t - \phi(S_t)\tran \theta\right) \mathop{\phi(S_t)}
    \,.
\end{align}
Our theorem below generalizes recent analysis of $1$-step TD learning \citep[][Theorem~2]{bhandari2018finite}.

\begin{restatable}[Finite-Time Analysis]{theorem}{theoremftanalysis}
    \label{theorem:ft_analysis}
    Suppose TD learning with linear function approximation is applied under an i.i.d.\ state model with stationary distribution $\smash{\vd \in \R^{|\mathcal{S}|}}$ (see \Cref{assump:iid} in \Cref{app:ft_analysis}) using the compound return estimator $G^\vc_t$ as its target.
    Let $\beta \in [0, 1)$ be the contraction modulus of the estimator and let $\sigma^2 \geq 0$ be the variance of the estimator.
    Assume that the features are normalized such that
    $\smash{\sqnorm{\phi(s)} \leq 1}$,
    ${\forall~s \in \mathcal{S}}$.
    Define
    $\smash{C \defeq (\norm{r}_\infty + (1+\gamma) \norm{\theta^*}_\infty) \mathbin{/} (1-\gamma)}$,
    where $\theta^*$ is the minimizer of the projected Bellman error for $G^\vc_t$.
    For any $T \geq (4 \mathbin{/} (1-\beta))^2$ and a constant step size $\alpha = 1 \mathbin{/} \sqrt{T}$,
    \begin{align*}
        &\expect{\norm{v_{\theta^*} - v_{\bar{\theta}_T}}_\mD^2}
        \leq \frac{\sqnorm{\theta^* - \theta_0} + 2 (1-\beta)^2 C^2 + 2 \sigma^2}{(1-\beta) \sqrt{T}}
        \,,\\
        &\qquad \text{where }
        \mD \defeq \diag(\vd)
        \text{ and }
        \bar{\theta}_T \defeq \frac{1}{T} \sum_{t=0}^{T-1} \theta_t
        \,.
    \end{align*}
\end{restatable}
\glue
\begin{proof}
    See \Cref{app:ft_analysis}.
\end{proof}
\glue

With a constant step size, compound TD learning (and hence $n$-step TD learning as a special case) reduces the value-function error at the same asymptotic rate of $O(1 / \sqrt{T})$ for any return estimator.
However, both the contraction modulus $\beta$ and the return variance $\sigma^2$ greatly influence the magnitude of the constant that multiplies this rate.
Given an $n$-step return and a compound return with the same contraction modulus, the compound return has lower variance by \Cref{theorem:variance_reduction} and therefore converges faster to its respective TD fixed point.
Although these two fixed points may generally be different, we can show that the bound on their solution quality is the same by generalizing Lemma~6 of \citet{tsitsiklis1997analysis} for an arbitrary return.

\begin{figure*}[t]
    \centering
    \includegraphics[width=0.24\textwidth]{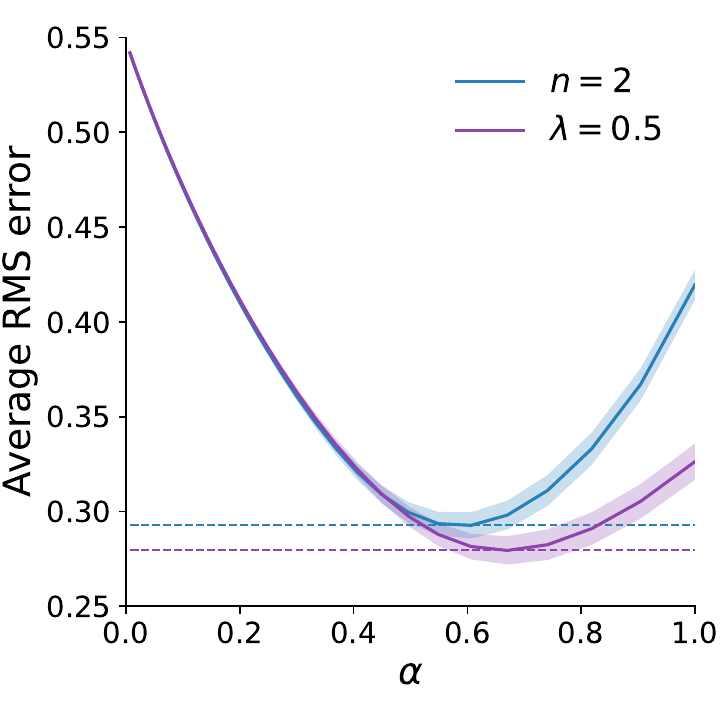}
    \hfill
    \includegraphics[width=0.24\textwidth]{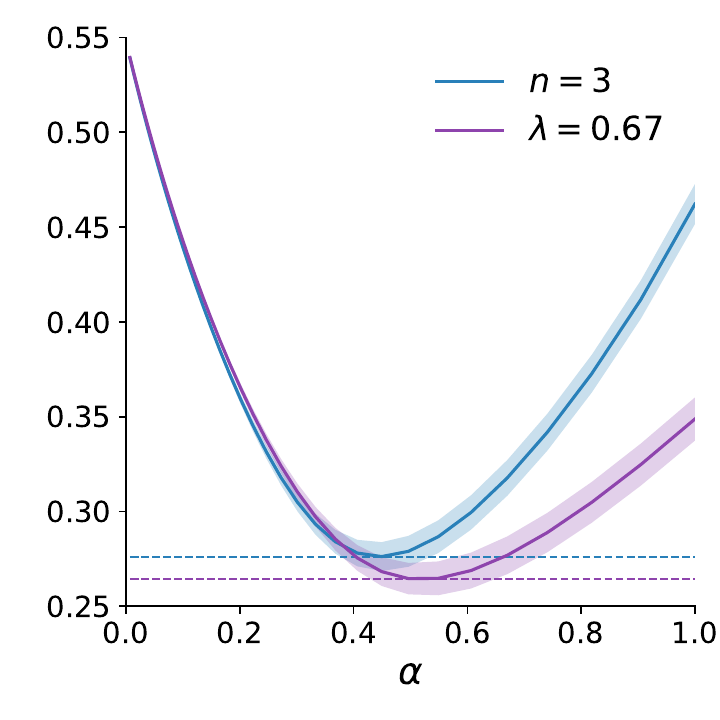}
    \hfill
    \includegraphics[width=0.24\textwidth]{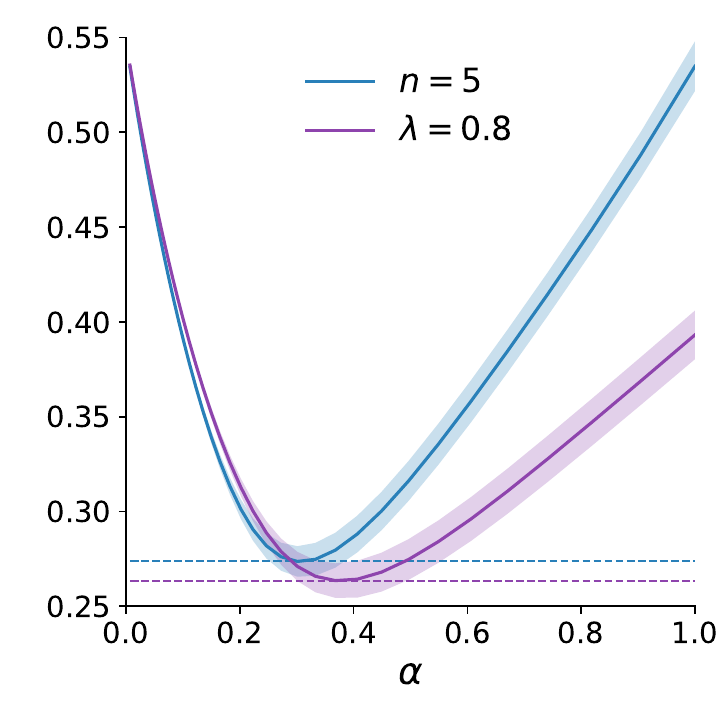}
    \hfill
    \includegraphics[width=0.24\textwidth]{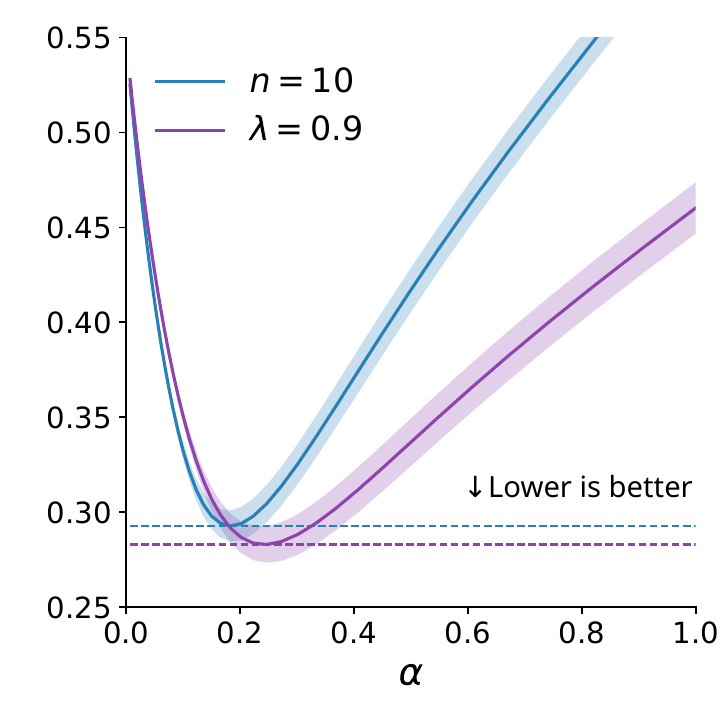}
    \captionvspace
    \caption{
        Comparing $n$-step returns and $\lambda$-returns, paired by COM, in a random walk.
        Dashed lines indicate the lowest errors attained.
    }
    \captionvspace
    \label{fig:rw19}
\end{figure*}

\begin{restatable}[TD Solution Quality]{proposition}{proptdsolutionquality}
    \label{prop:td_solution_quality}
    Let $\theta^*$ be the minimizer of the projected Bellman error under linear function approximation for any $n$-step or compound return with contraction modulus $\beta$.
    The following bound always holds:
    \begin{equation*}
        \norm{\mPhi \theta^* - v_\pi}_\mD \leq \frac{1}{1-\beta} \norm{\mPi v_\pi - v_\pi}_\mD
        \,.
    \end{equation*}
\end{restatable}
\glue
\begin{proof}
    See \Cref{app:prop_td_solution_quality}.
\end{proof}
\glue

This shows that the fixed-point error for an arbitrary return is always within $1 \mathbin{/} (1-\beta)$ times the optimal error.
Since $\beta$ is equalized for the two returns in \Cref{theorem:ft_analysis}, the compound return converges to its fixed point faster than the $n$-step return converges to its fixed point, and the quality bound of these two fixed points is the same.

\subsection{Case Study: \texorpdfstring{$\lambda$}{λ}-returns}
\label{subsec:lambda_case_study}

\halfglue

Although the $\lambda$-return is often motivated by its efficient implementation using TD($\lambda$) and eligibility traces, our theory indicates that $\lambda$-returns can also promote faster learning via variance reduction.
We provide empirical support for this by demonstrating faster learning in the random-walk experiment from \citet[][Sec.~12.1]{sutton2018reinforcement}.
In this environment, the agent begins in the center of a linear chain of 19 connected states and can move either left or right.
The agent receives a reward only if it reaches one of the far ends of the chain ($-1$ for the left, $+1$ for the right), in which case the episode terminates.
The agent's policy randomly moves in either direction with equal probability.
We train the agents for 10 episodes, updating the value functions after each episode with offline backups like \Cref{eq:backup}.
To pair the $n$-step returns and $\lambda$-returns together, we derive the effective $\lambda$ for an $n$-step return in the following proposition.

\begin{restatable}[Effective $\lambda$ of $n$-step return]{proposition}{propeffectivelambda}
    \label{prop:lambda_effective_nstep}
    For any ${n \geq 1}$, when $\gamma < 1$, the $\lambda$-return with
    $\lambda = (1 - \gamma^{n-1}) \mathbin{/} (1 - \gamma^n)$
    has the same contraction modulus as the $n$-step return.
    When $\gamma = 1$, the $\lambda$-return with $\lambda = (n-1) \mathbin{/} n$ has the same COM as the $n$-step return.
\end{restatable}
\glue
\begin{proof}
    See \Cref{app:prop_lambda_effective_nstep}.
\end{proof}
\glue

Because this is an undiscounted task, we use the relationship\footnote{
    Another way to write this relationship is $n = 1 \mathbin{/} (1 - \lambda)$, which makes it clear how the effective $n$-step is affected by $\lambda$.
}
$\lambda = {(n-1) \mathbin{/} n}$ to generate several $(n,\lambda)$-pairs with equal COMs in \Cref{tab:alr}.
For our experiment, we choose four commonly used $n$-step values, $\{2,3,5,10\}$, which correspond to $\lambda \in \{0.5, 0.67, 0.8, 0.9\}$.
In \Cref{fig:rw19}, we plot the average root-mean-square (RMS) value error (with respect to $v_\pi$) as a function of the step size $\alpha$ over 100 trials.
We also indicate 95\% confidence intervals by the shaded regions.
For all of the tested $(n,\lambda)$-pairs, an interesting trend emerges.
In the left half of each plot, variance is not an issue because the step size is small and has plenty of time to average out the randomness in the estimates.
Learning therefore progresses at a nearly identical rate for both the $n$-step return and the $\lambda$-return since they have the same COM (although there is a small discrepancy due to the truncation of the episodic task).
However, as $\alpha \to 1$ in the right half of each plot, variance becomes a significant factor as the step size becomes too large to mitigate the noise in the updates.
This causes the $n$-step return's error to diverge sharply compared to the $\lambda$-return's as the $\lambda$-return manages variance more effectively.
The lowest error attained by each $\lambda$-return is also better than that of its corresponding $n$-step return in all cases.
Notably, neither of our variance assumptions hold perfectly in this environment, demonstrating that our variance model's predictions are useful in practical settings.

\begin{table}[t]
    \centering
    \caption{Common $n$-step returns and $\lambda$-returns with equal COMs.}
    \label{tab:alr}
    \vspace{0.1in}
    \begin{tabular}{l|rrrrrrrrr}
        \toprule
        $n$ & 2 & 3 & 4 & 5 & 10 & 20 & 50 & 100 \\
        \midrule
        $\lambda$ & 0.5 & 0.67 & 0.75 & 0.8 & 0.9 & 0.95 & 0.98 & 0.99 \\
        \bottomrule
    \end{tabular}
    \glue
\end{table}

\section{Piecewise \texorpdfstring{$\lambda$}{λ}-Returns}
\label{sec:approximate}

Although the previous experiment shows that $\lambda$-returns can achieve faster learning, they remain expensive for deep RL.
This is because the $\lambda$-return at time $t$ theoretically bootstraps on every time step after $t$ (until the end of an episode), with each bootstrap requiring a forward pass through the neural network.
Even when truncating the $\lambda$-return to a reasonable length, this is still several times more expensive than the single bootstrap of an $n$-step return.
Previous work has amortized the cost of $\lambda$-returns over long trajectories by exploiting their recursive structure \citep[e.g.,][]{munos2016safe,harb2016investigating,daley2019reconciling},
but the price to pay for this efficiency is the requirement that experiences must be temporally adjacent, which can hurt performance.
Our preliminary experiments confirmed this, indicating that the correlations within replayed trajectories counteract the benefits of $\lambda$-returns when compared to minibatches of $n$-step returns (with the batch size chosen to equalize computation).

We instead seek a compound return that approximates the variance-reduction property of the $\lambda$-return while being computationally efficient for minibatch replay.
There are many ways we could average $n$-step returns together, and so we constrain our search by considering compound returns that
1)~comprise an average of only two $n$-step returns to minimize cost,
2)~preserve the contraction modulus or COM of the $\lambda$-return, and
3)~place weights on the TD errors that are close to those assigned by the $\lambda$-return.

The first property constrains our estimator to have the two-bootstrap form of \Cref{eq:two-bootstrap}:
$\smash{(1-c) \nstep{n_1}_t + c \nstep{n_2}_t}$.
Let $n$ be our targeted effective $n$-step;
the effective $\lambda$ can be obtained from \Cref{prop:lambda_effective_nstep}.
Let us also assume $\gamma < 1$, although we include the case where $\gamma=1$ in \Cref{app:pilar}.
To preserve the contraction modulus as in the second property, we must satisfy
$(1-c) \gamma^{n_1} + c \gamma^{n_2} = \gamma^n$.
Assuming that we have freedom in the choice of $n_1$ and $n_2$, it follows that
$c = (\gamma^n - \gamma^{n_1}) \mathbin{/} (\gamma^{n_2} - \gamma^{n_1})$\,.

\begin{wrapfigure}{R}{0.5\columnwidth}
    \centering
    \includegraphics[width=0.5\columnwidth]{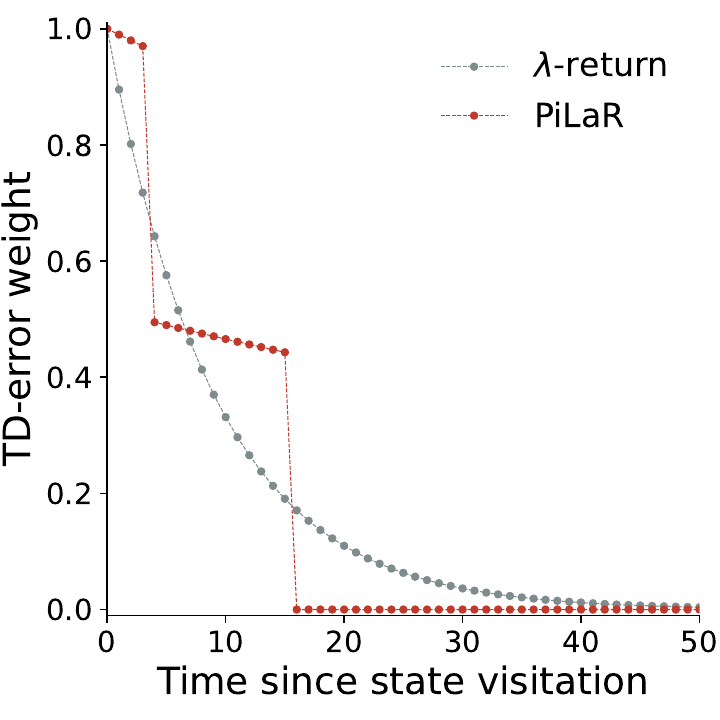}
    \captionvspace
    \caption{
        TD-error weights for Pilar and a $\lambda$-return (${\lambda=0.904}$).
        Both returns have the same contraction modulus as a 10-step return when $\gamma=0.99$.
    }
    \label{fig:pilar}
\end{wrapfigure}

We would thus like to find $n_1$ and $n_2$ such that the weights given to the TD errors optimize some notion of closeness to the TD($\lambda$) weights, in order to fulfill the third and final property.
Although there are many ways we could define the error, we propose to minimize the maximum absolute difference between the weights, since this ensures that no individual weight deviates too far from the TD($\lambda$) weight.
Recall that the cumulative weight given to TD error $\delta_{t+i}$ by $n$-step return $\smash{\nstep{n}_t}$ is $h_i = 1$ if $i < n$, and is $h_i = 0$ otherwise.
It follows that the two-bootstrap average assigns a weight of ${h_i = 1}$ if ${i < n_1}$;
a weight of ${h_i = c}$ if ${n_1 \leq i < n_2}$ since ${(1-c) \cdot 0 + c \cdot 1 = c}$;
and a weight of $h_i = 0$ otherwise.
We then minimize the error
$\max_{i \geq 0} \abs{\gamma^i h_i - (\gamma \lambda)^i}$.

We call our approximation Piecewise $\lambda$-Return because each weight $h_i$ is a piecewise function whose value depends on where $i$ lies in relation to the interval $[n_1,n_2)$.
\Cref{fig:pilar} illustrates how Pilar roughly approximates the TD($\lambda$) decay using a step-like shape.
Although Pilar's TD-error weights do not form a smooth curve, they retain important properties like contraction modulus, monotonicity, and variance reduction.
Crucially, Pilar is much cheaper to compute than the $\lambda$-return, making it more suitable for minibatch experience replay.
In \Cref{app:pilar}, we describe a basic search algorithm for finding the lowest-error $(n_1,n_2)$-pair, along with a reference table of precomputed Pilars for $\gamma=0.99$.

\section{Deep RL Experiments}
\label{sec:experiments}

We consider a multistep generalization of  Deep Q-Network \citep[DQN;][]{mnih2015human}.
The action-value function $q(s,a;\theta)$ is implemented as a neural network to enable generalization over high-dimensional states, where $\theta \in \mathbb{R}^d$ is the learnable parameters.
A stale copy $\theta^-$ of the parameters is used only for bootstrapping and is infrequently updated from $\theta$ in order to stabilize learning.
The agent interacts with its environment and stores each transition $(s,a,r,s')$---where $s$ is the state, $a$ is the action, $r$ is the reward, and $s'$ is the next state---in a replay memory $\mathcal{D}$.
The network's loss function is defined as
\begin{equation*}
    \label{eq:dqn_multistep}
    \mathcal{L}(\theta,\theta^-)
    \defeq \mathbb{E} \! \left[ \frac{1}{2} \left( \hat{G}(r,s',\dots;\theta^-) - q(s,a;\theta) \right)^2 \right]
    \,,
\end{equation*}
where $\hat{G}(r,s',\dots;\theta^-)$ is a return estimator for $q_*(s,a)$ and the expectation is taken over a uniform distribution on~$\mathcal{D}$.
The network is trained by minimizing this loss via stochastic gradient descent on sampled minibatches.
\looseness=-1

\begin{figure*}[t]
    \centering
    \includegraphics[width=0.19\textwidth]{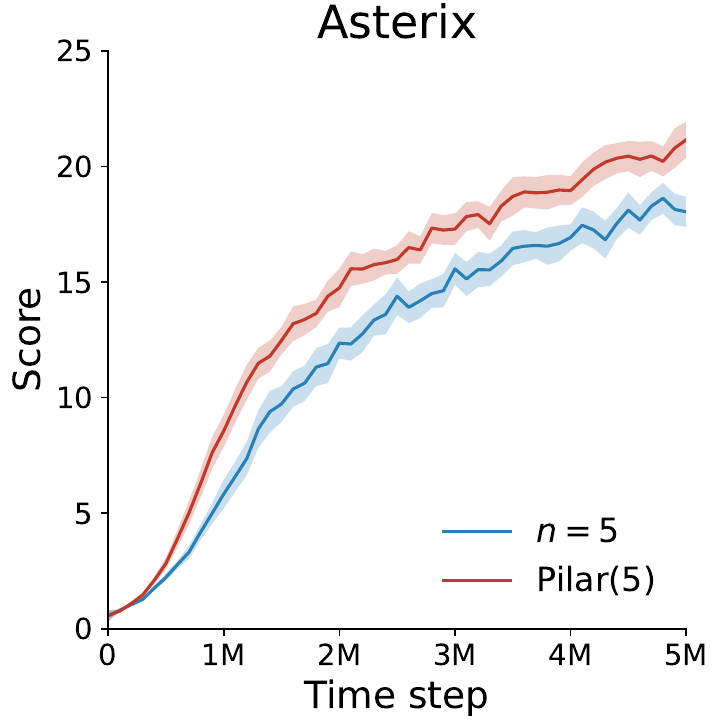}
    \hfill
    \includegraphics[width=0.19\textwidth]{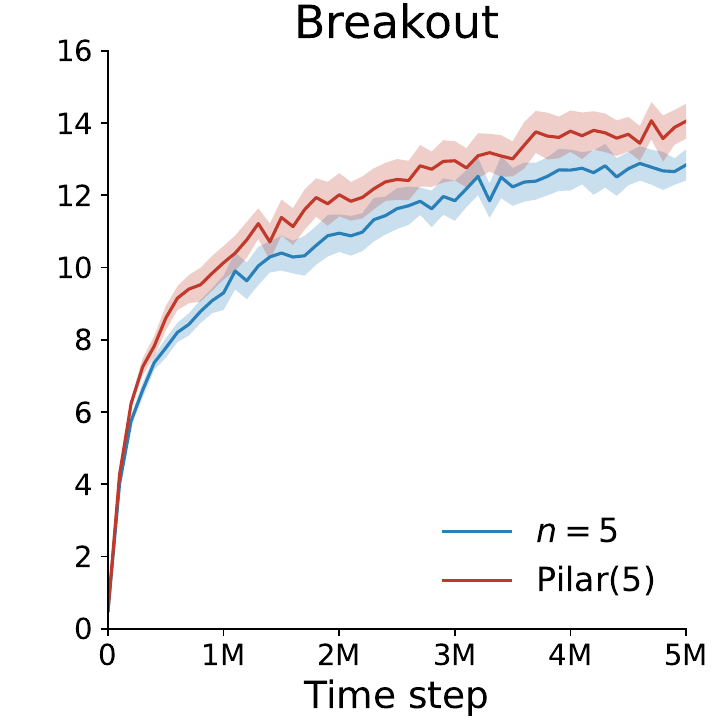}
    \hfill
    \includegraphics[width=0.19\textwidth]{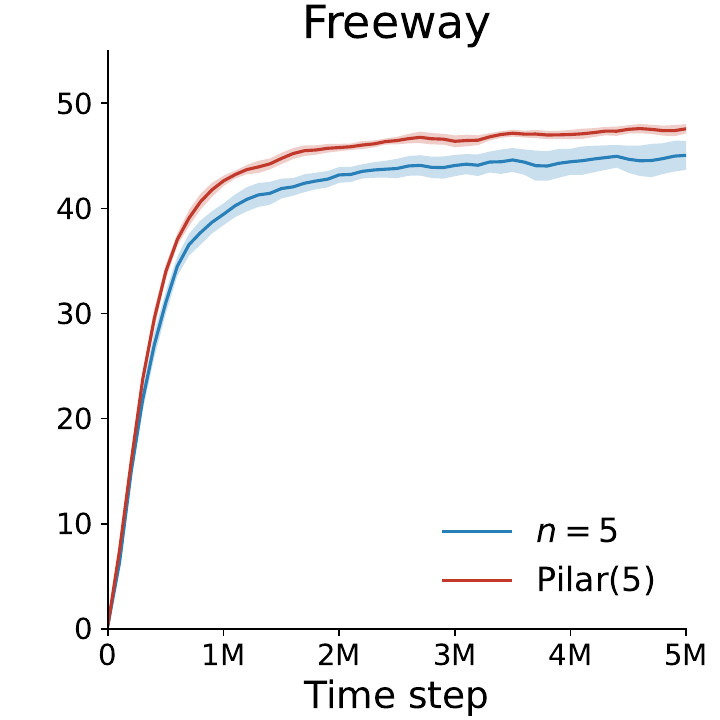}
    \hfill
    \includegraphics[width=0.19\textwidth]{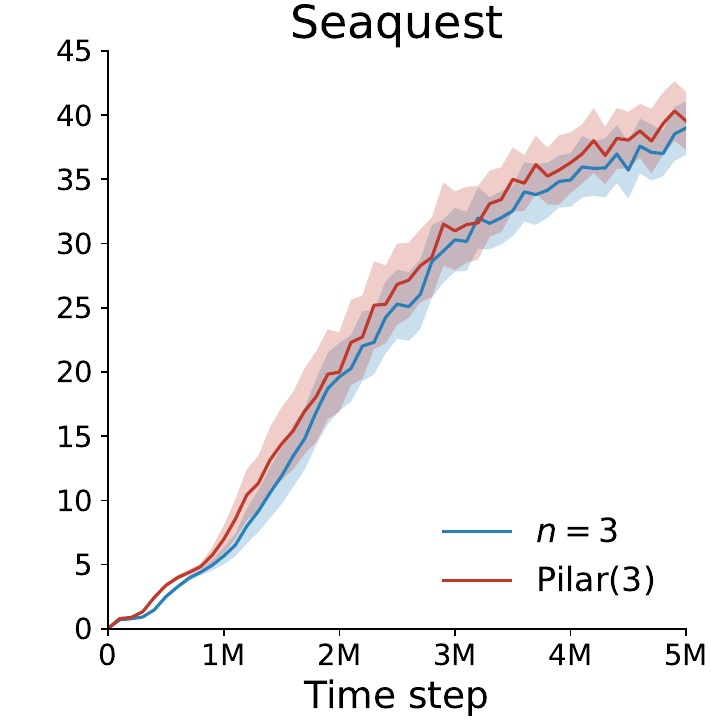}
    \hfill
    \includegraphics[width=0.19\textwidth]{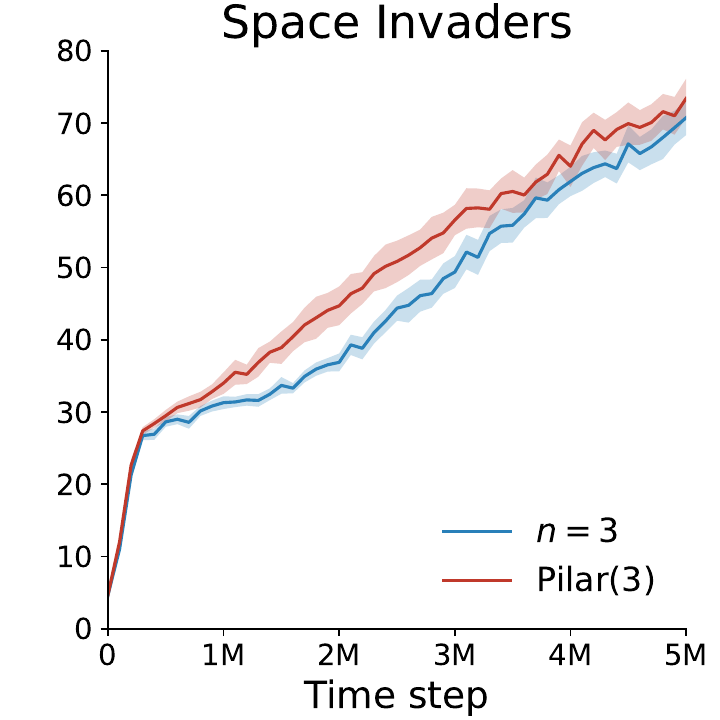}
    \captionvspace
    \caption{
        Learning curves for DQN with $n$-step returns and Pilars in five MinAtar games.
    }
    \label{fig:lcurves_dqn_main}
\end{figure*}

\begin{figure*}[t]
    \centering
    \includegraphics[width=0.19\textwidth]{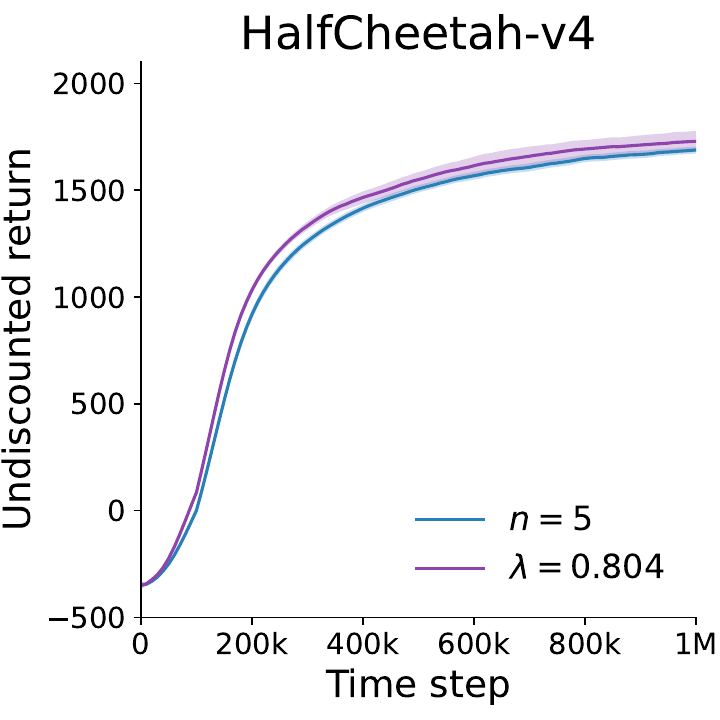}
    \hfill
    \includegraphics[width=0.19\textwidth]{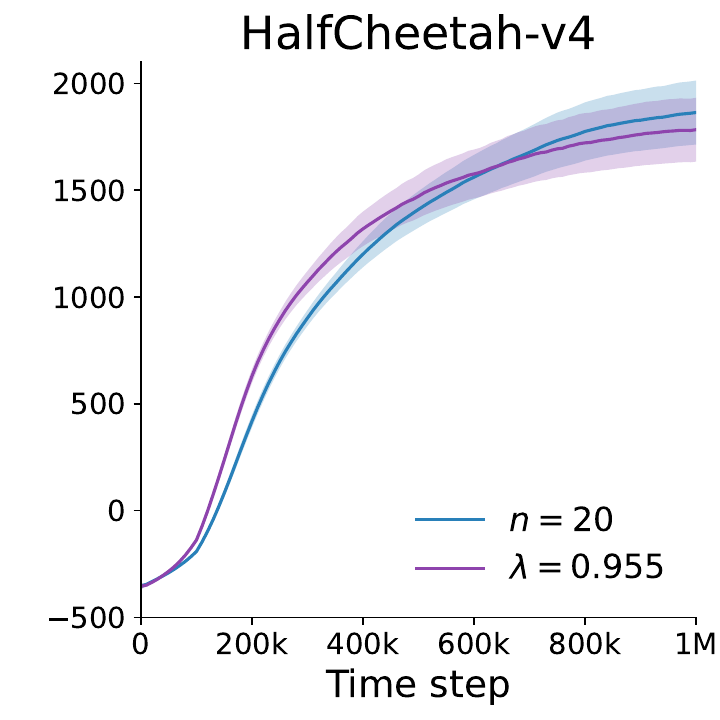}
    \hfill
    \includegraphics[width=0.19\textwidth]{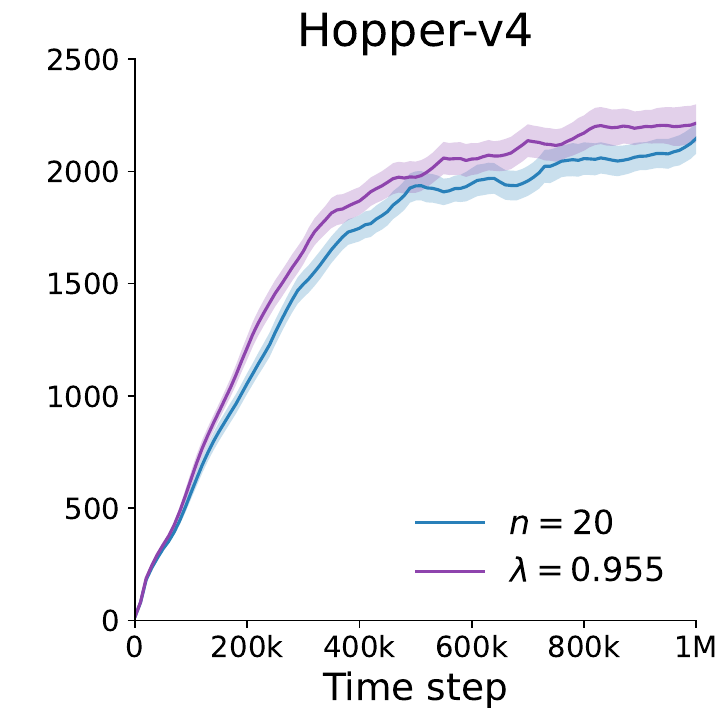}
    \hfill
    \includegraphics[width=0.19\textwidth]{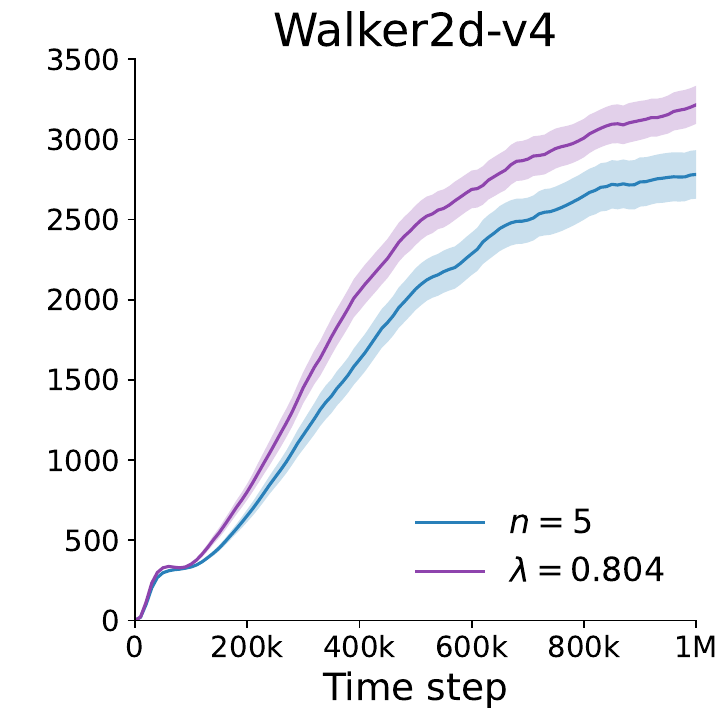}
    \hfill
    \includegraphics[width=0.19\textwidth]{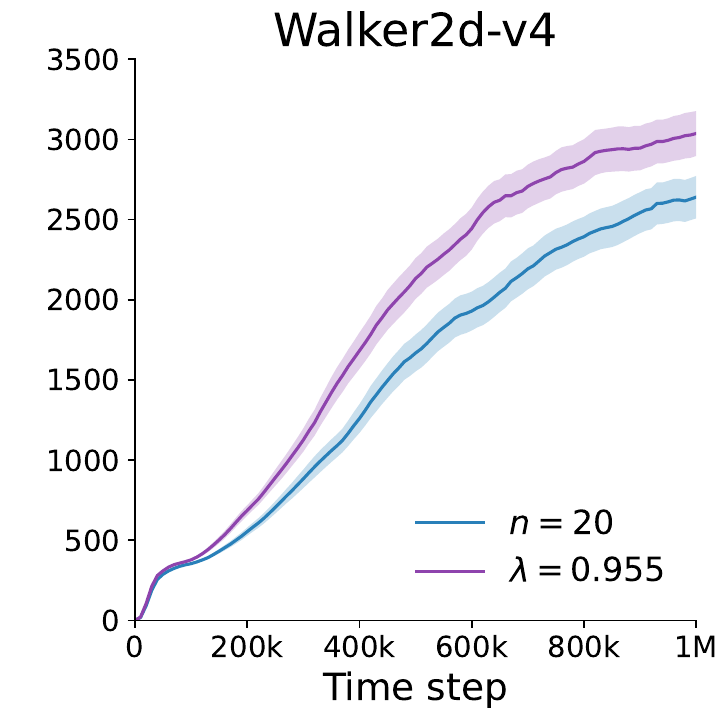}
    \captionvspace
    \caption{
        Learning curves for PPO with $n$-step returns and $\lambda$-returns in three MuJoCo environments.
    }
    \label{fig:lcurves_ppo_main}
\end{figure*}

\begin{wraptable}{R}{0.5\columnwidth}
    \glue
    \glue
    \centering
    \caption{Pilars used for MinAtar.}
    \vspace{0.1in}
    \label{tab:experiment_pilars}
    \begin{tabular}{r|lll}
        \toprule
        $n$-step & $n_1$ & $n_2$ & $c$ \\
        \midrule
        3 & 1 & 6 & 0.406 \\
        5 & 2 & 9 & 0.437 \\
        \bottomrule
    \end{tabular}
    \glue
\end{wraptable}

We test our agents in all five MinAtar games \citep{young2019minatar}:
Asterix, Breakout, Freeway, Seaquest, and Space Invaders.
The states are $10\times 10$ multi-channel images depicting object locations and velocity trails.
Each agent's network is a two-layer convolutional architecture with rectified linear units (ReLUs).
The agents execute an $\epsilon$-greedy policy for 5M time steps (where $\epsilon$ is linearly annealed from $1$ to $0.1$ over the first 100k steps) and conduct a minibatch update of 32 samples on every step.
We provide more details in \Cref{app:experiment_setup}.
Code is available online.\footnote{
    {\tiny \url{https://github.com/brett-daley/averaging-nstep-returns}}
}

For $n = 3$ and $n = 5$, we compare the $n$-step return against the corresponding Pilar of the same contraction modulus with the given discount factor, $\gamma = 0.99$ (see \Cref{tab:experiment_pilars} for specific values of $n_1$, $n_2$, and $c$).
We chose five Adam step sizes over a logarithmic grid search and generated learning curves by plotting the 100-episode moving average of undiscounted return (game score) versus time steps.
We then chose the step size for each game-estimator pair that maximized the area under the curve over the last 1M steps, or 20\% of training.
We showcase the best results for each game in \Cref{fig:lcurves_dqn_main} but include all of the learning curves in \Cref{app:experiment_setup}.
Out of the five games, Seaquest was the only one in which Pilar did not outperform $n$-step returns in a statistically significant way.
The results are averaged over 32 trials, with shading to indicate 95\% confidence intervals.

In several cases, Pilars are able to significantly improve the agent's best performance compared to $n$-step returns.
This is in spite of the fact that each pair of $n$-step return and Pilar have the same contraction modulus, suggesting that the performance increase is partially due to variance reduction.
Overall, Pilars outperform $n$-step returns for 1 in 5 games (20\%) when $n=3$, and 3 in 5 games (60\%) when $n=5$.
In the remaining cases, the performance is not significantly different;
thus, Pilars improve average performance in these games.
Because the average performance gap between the estimators also widens as $n$ increases, Pilar's benefits appears to become more pronounced for longer---and, hence, higher-variance---$n$-step returns.

These results corroborate our theory by showing that averaging $n$-step returns can accelerate learning, even with a nontrivial network architecture.
We do not claim that Pilars are necessarily the best average to achieve variance reduction, but our experiments still demonstrate the practical relevance of such compound returns.

\halfglue

\paragraph{PPO}
To demonstrate that our theory also applies to on-policy methods such as Proximal Policy Optimization \citep[PPO;][]{schulman2017proximal}, we conduct additional experiments in three MuJoCo environments \citep{todorov2012mujoco}: Half Cheetah, Hopper, and Walker~2D.
Because on-policy agents like PPO are trained on relatively short trajectories of recent experiences, it is feasible to compute exact $\lambda$-returns for every experience, often referred to as Generalized Advantage Estimation \citep[GAE;][]{schulman2015high} in this context.
This means that an approximation like Pilar is not necessary in this setting and we can directly compare $\lambda$-returns with $n$-step returns.

We use the CleanRL implementation of PPO \citep{huang2022cleanrl}.
We compare different returns with the values $(n, \lambda) \in {(5, 0.804), (10, 0.905), (20, 0.955)}$.
These pairs were chosen because the respective $\lambda$-returns and $n$-step returns have the same contraction moduli when $\gamma=0.99$.
We specifically include $n=20$ because $\lambda=0.95$ is a common default value for PPO.
We showcase the best results in \Cref{fig:lcurves_ppo_main} but include all of the learning curves in \Cref{app:experiment_setup}.
The results are averaged over 100 trials, with shading to indicate 95\% confidence intervals.
In seven out of nine of the cases, $\lambda$-returns significantly improve sample efficiency over $n$-step returns with respect to the 100-episode moving average of undiscounted return, again demonstrating the benefits of compound returns.

\section{Conclusion}

\halfglue

We have shown that compound returns, including $\lambda$-returns, have a variance-reduction property.
This is the first evidence, to our knowledge, that $\lambda$-returns have a theoretical learning advantage over $n$-step returns in the absence of function approximation;
it was previously believed that both were different yet equivalent ways of interpolating between TD and Monte Carlo learning.
Our random-walk experiments confirm that an appropriately chosen $\lambda$-return performs better than a given $n$-step return across a range of step sizes.
In replay-based deep RL methods like DQN, where $\lambda$-returns are difficult to implement efficiently, we demonstrated with Pilar that a simpler average is still able to train neural networks with fewer samples.
Since the average is formed from only two $n$-step returns, the additional computational cost is negligible---less expensive than adding a second target network, as is often done in recent methods \citep[e.g.,][]{fujimoto2018addressing,haarnoja2018soft}.
\looseness=-1

Although we are able to establish strong theoretical guarantees regarding the convergence rate of multistep TD learning with linear function approximation, we note that this result does not automatically extend to nonlinear approximators.
This is not a limitation of our analysis, but rather of the class of TD algorithms commonly used in deep RL.
Indeed, even $1$-step TD does not converge with certain nonlinear function approximators \citep[see, e.g.,][Sec.~10]{tsitsiklis1997analysis}.
Although our experiments have shown that compound returns often do improve performance for DQN and PPO, it is not guaranteed that these results will generalize to other methods.
Additionally, we did not show that the improved performance for these algorithms is due solely to variance reduction;
it is possible that other favorable properties of the compound returns also contributed to this improvement, such as convergence to a different fixed point (despite the bound that we proved in \Cref{prop:td_solution_quality}) or, as we show in concurrent work, the long-tailed credit-assignment characteristic of certain compound returns \citep{daley2024demystifying}.
The added complexities of nonconvex optimization likely explain why the performance trends for DQN and PPO are not always as clear as our random-walk results.
Nevertheless, we note that the variance-reduction property \emph{is} a general phenomenon, since the TD-error assumptions are agnostic to the choice of function approximator.
This property, coupled with our empirical results, indicates potential for using compound returns in other deep RL methods to improve sample efficiency.
\looseness=-1

A number of interesting extensions to our work are possible.
For instance, we derived Pilar under the assumption that the $\lambda$-return is a good estimator to approximate.
However, the exponential decay of the $\lambda$-return originally arose from the need for an efficient online update rule using eligibility traces, and is not necessarily optimal in terms of the bias-variance trade-off.
With experience replay, we are free to average $n$-step returns in any way we want, even if the average would not be easily implemented online.
This opens up exciting possibilities for new families of return estimators:
e.g., those that minimize variance for a given contraction modulus or COM.
Based on our compound variance model (\Cref{prop:compound_variance}), a promising direction in this regard appears to be weights that initially decay faster than the exponential function, but then slower afterwards.
Minimizing variance becomes even more important for off-policy learning, where the inclusion of importance-sampling ratios greatly exacerbates variance.
Recent works \citep{munos2016safe,daley2023trajectory} have expressed arbitrary off-policy corrections in terms of weighted sums of TD errors, and so our theory could be extended to this setting with only minor modifications.

\section*{Acknowledgments}

\halfglue

This research is supported in part by the Natural Sciences and Engineering Research Council of Canada (NSERC), the Canada CIFAR AI Chair Program, and the Digital Research Alliance of Canada.

\section*{Impact Statement}

This paper presents work whose goal is to advance the field of Machine Learning.
There are many potential societal consequences of our work, none which we feel must be specifically highlighted here.

\bibliography{icml2024}
\bibliographystyle{icml2024}

%%%%%%%%%%%%%%%%%%%%%%%%%%%%%%%%%%%%%%%%%%%%%%%%%%%%%%%%%%%%%%%%%%%%%%%%%%%%%%%
%%%%%%%%%%%%%%%%%%%%%%%%%%%%%%%%%%%%%%%%%%%%%%%%%%%%%%%%%%%%%%%%%%%%%%%%%%%%%%%
% APPENDIX
%%%%%%%%%%%%%%%%%%%%%%%%%%%%%%%%%%%%%%%%%%%%%%%%%%%%%%%%%%%%%%%%%%%%%%%%%%%%%%%
%%%%%%%%%%%%%%%%%%%%%%%%%%%%%%%%%%%%%%%%%%%%%%%%%%%%%%%%%%%%%%%%%%%%%%%%%%%%%%%
\newpage
\appendix
\onecolumn
\appendix

\section{Variance Assumptions}
\label{app:assumptions}

Our variance assumptions are a significant relaxation and generalization of the assumptions made by the $n$-step variance model of \citet[][Sec.~3]{konidaris2011td_gamma}.
We show this by starting from their original assumptions and then describing the steps taken to obtain our new assumptions.

\citeauthor{konidaris2011td_gamma} begin by expanding the variance of an $n$-step return in the following way:
\begin{align}
    \nonumber
    \Var[\nstep{n}_t \mid S_t]
    &= \Var[\nstep{n-1}_t + \gamma^{n-1} \delta_{t+n-1} \mid S_t] \\
    \label{eq:nstep_variance_expanded}
    &= \Var[\nstep{n-1}_t \mid S_t] + \gamma^{2(n-1)} \Var[\delta_{t+n-1} \mid S_t] + 2 \; \Cov[\nstep{n-1}_t, \gamma^{n-1} \delta_{t+n-1} \mid S_t]
    \,.
\end{align}
The covariance term is equivalent to
\begin{align}
    \nonumber
    \Cov[\nstep{n-1}_t, \gamma^{n-1} \delta_{t+n-1} \mid S_t]
    &= \Cov[\nstep{n-1}_t,\nstep{n}_t - \nstep{n-1}_t \mid S_t] \\
    \nonumber
    &= \Cov[\nstep{n-1}_t,\nstep{n}_t \mid S_t] - \Cov[\nstep{n-1}_t,\nstep{n-1}_t \mid S_t] \\
    \label{eq:konidaris_diff}
    &= \Cov[\nstep{n-1}_t,\nstep{n}_t \mid S_t] - \Var[\nstep{n-1}_t \mid S_t]
    \,.
\end{align}
With the rationale that $\smash{\nstep{n-1}_t}$ and $\smash{\nstep{n}_t}$ are generated from the same trajectory and therefore highly correlated, the authors assume that
\begin{equation}
    \label{eq:konidaris_approx_cov}
    \Cov[\nstep{n-1}_t,\nstep{n}_t \mid S_t] \approx \Var[\nstep{n-1}_t \mid S_t]
    \,,
\end{equation}
which makes \Cref{eq:konidaris_diff} approximately zero.
Consequently, \Cref{eq:nstep_variance_expanded} becomes
\begin{equation}
    \label{eq:konidaris_nstep_var_recursive}
    \Var[\nstep{n}_t \mid S_t] \approx \Var[\nstep{n-1}_t \mid S_t] + \gamma^{2(n-1)} \Var[\delta_{t+n-1} \mid S_t]
    \,.
\end{equation}
The authors additionally assume that the variance of each TD error is the same:

\tdvar*

Hence, \Cref{eq:konidaris_nstep_var_recursive} becomes
$\Var[\nstep{n}_t \mid S_t] \approx \Var[\nstep{n-1}_t \mid S_t] + \gamma^{2(n-1)} \kappa$\,.
Since $\kappa = \Var[\delta_t \mid S_t] = \Var[\nstep{1}_t \mid S_t]$, unrolling the recursion gives the final $n$-step variance model:
\begin{equation}
    \label{eq:konidaris_nstep_model}
    \Var[\nstep{n}_t \mid S_t] \approx \sum_{i=0}^{n-1} \gamma^{2i} \kappa
    \,.
\end{equation}
This completes the derivation from \citet[][Sec.~3]{konidaris2011td_gamma}, where the two major assumptions are \Cref{eq:konidaris_approx_cov,assump:td_var}.
Notice, however, that \Cref{eq:konidaris_nstep_model} could be obtained more simply by assuming that TD errors are uncorrelated---in fact, this is equivalent to assuming \Cref{eq:konidaris_approx_cov} holds.
One way to see this is by decomposing the $n$-step return into a sum of TD errors and applying standard variance rules, assuming no correlations between the random variables:
\begin{equation*}
    \Var[\nstep{n}_t \mid S_t]
    = \Var[\nstep{n}_t - v(S_t) \mid S_t]
    = \Var\!\left[\sum_{i=0}^{n-1} \gamma^i \delta_{t+i} \Biggm| S_t\right]
    = \sum_{i=0}^{n-1} \gamma^{2i} \Var[\delta_{t+i} \mid S_t]
    = \sum_{i=0}^{n-1} \gamma^{2i} \kappa
    \,,
\end{equation*}
which is identical to \Cref{eq:konidaris_nstep_model}.
In our work, we directly make this uniform-correlation assumption, while also generalizing it with an arbitrary correlation coefficient $\rho \in [0,1]$.
Note that we must have ${\rho \geq 0}$ because an infinite number of TD errors cannot be negatively correlated simultaneously.

\tdcovar*

To summarize, our variance model makes \Cref{assump:td_var,assump:td_cov}, which are equivalent to the assumptions made by \citet{konidaris2011td_gamma} when $\rho = 0$.
Since
$\smash{\Var[\delta_{t+i} \mid S_t] = \Cov[\delta_{t+i},\delta_{t+i} \mid S_t]}$,
we are able to combine \Cref{assump:td_var,assump:td_cov} into a single, concise expression---see \Cref{eq:assumps_combined} in \Cref{sec:analysis}.

Although the assumptions here may not always hold in practice, they would be difficult to improve without invoking information about the MDP's transition function or reward function.
We show in \Cref{app:how_real_assumptions} that our derived variances based on these assumptions are nevertheless consistent with empirical data in real environments.

\clearpage
\section{How Realistic Is the Proposed Variance Model?}
\label{app:how_real_assumptions}

Our assumptions state that all TD errors have uniform variance and are equally correlated to one another.
Since these assumptions may be violated in practice, it is informative to test how well our $n$-step variance model (\Cref{prop:nstep_variance}) compares to the true $n$-step variances in several examples.

We consider three environments:
the 19-state random walk (see \Cref{sec:analysis}), a $4 \times 3$ gridworld \citep[][Fig.~17.1]{russell2010artificial}, and a $10 \times 8$ gridworld \citep[][Fig.~7.4]{sutton2018reinforcement}.
We choose these environments because they have known dynamics and are small enough to exactly calculate $v_\pi$ with dynamic programming.
The two gridworlds are stochastic because each of the four moves (up, down, left, right) succeeds with only probability 80\%;
otherwise, the move is rotated by 90 degrees in either direction with probability 10\% each.
We let the agent execute a uniform-random behavior policy for all environments.

To make the results agnostic to any particular learning algorithm, we use $v_\pi$ to compute the TD errors.
We apply a discount factor of $\gamma = 0.99$ to the $10 \times 8$ gridworld (otherwise $v_\pi(s)$ would be constant for all $s$ due to the single nonzero reward) and leave the other two environments undiscounted.
We then measure the variance of the $n$-step returns originating from the initial state of each environment, for $n \in \{1,\dots,21\}$.
\Cref{fig:nstep_variance} shows these variances plotted as a function of $n$ and averaged over 10k episodes.
The best-case (optimistic, $\rho=0$) and worst-case (pessimistic, $\rho=1$) variances predicted by the $n$-step model, assuming that $\kappa = \Var[\delta_0 \mid S_0]$, are also indicated by dashed lines.

For all of the environments, the measured $n$-step variances always remain within the lower and upper bounds predicted by \Cref{prop:nstep_variance}.
These results show that our $n$-step variance model can still make useful variance predictions even when our assumptions do not hold.
The variances also grow roughly linearly as a function of $n$, corresponding more closely to the linear behavior of the optimistic, uncorrelated case than the quadratic behavior of the pessimistic, maximally correlated case.
This further suggests that the majority of TD-error pairs are weakly correlated in practice, which makes sense because temporally distant pairs are unlikely to be strongly related.

\begin{figure}[h]
    \centering
    \includegraphics[width=0.32\textwidth]{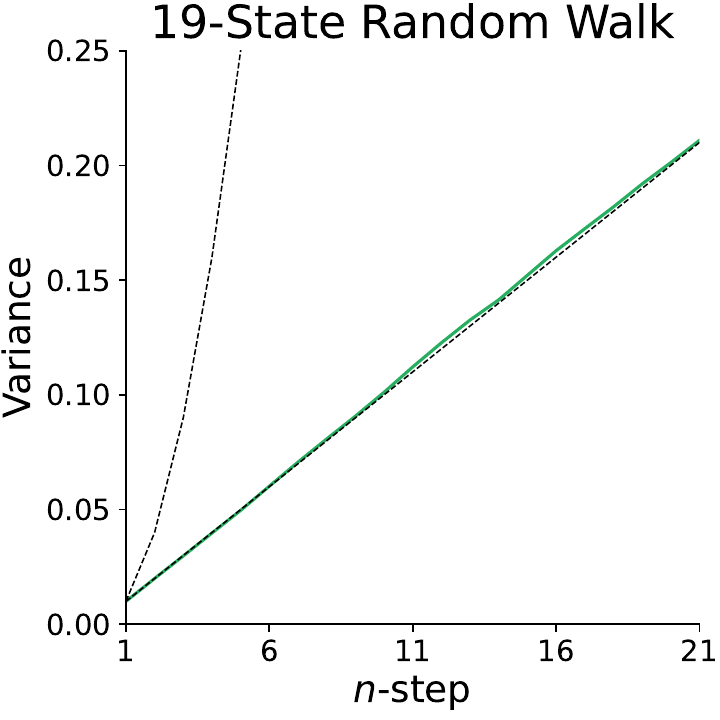}
    \hfill
    \includegraphics[width=0.32\textwidth]{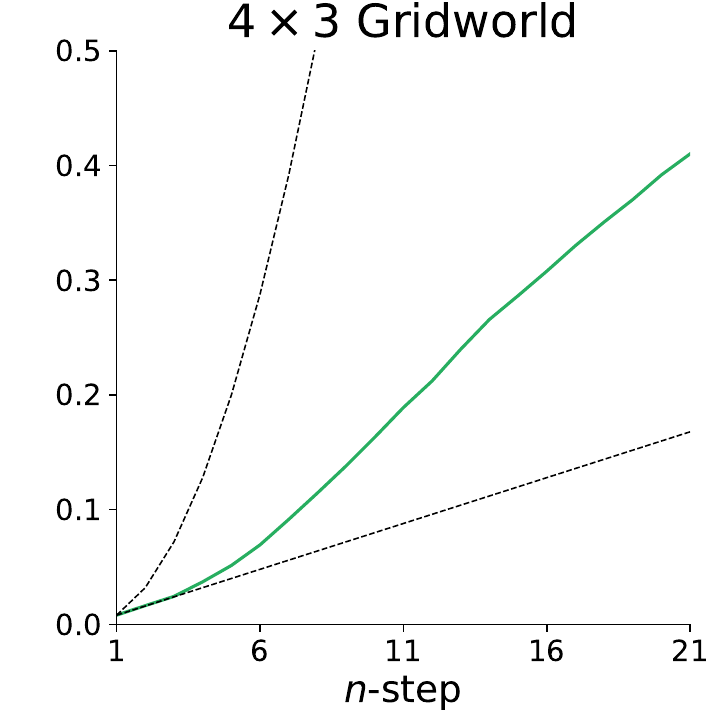}
    \hfill
    \includegraphics[width=0.32\textwidth]{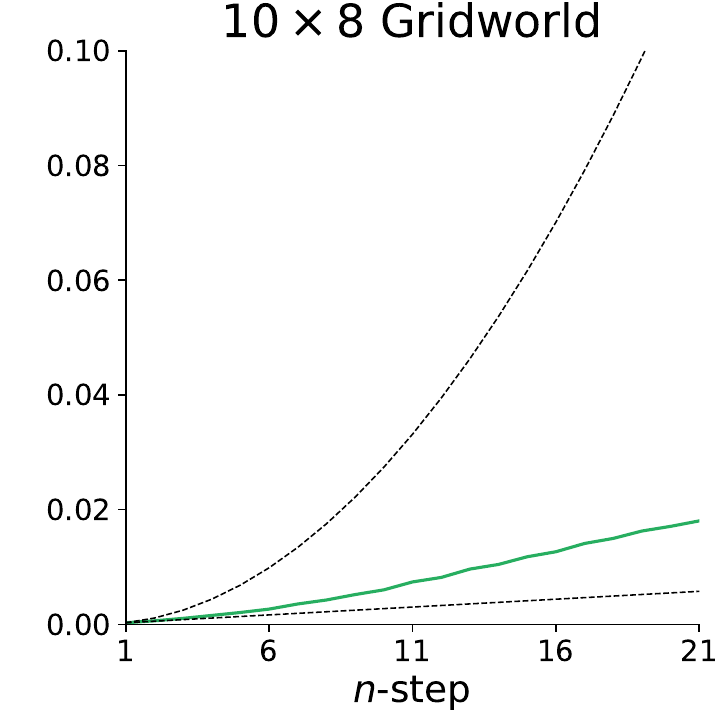}
    \captionvspace
    \caption{
        Variances of the $n$-step returns originating from the initial state in three environments.
        The solid green line indicates the true variance while the dashed black lines indicate the lower and upper bounds predicted by our $n$-step variance model (\Cref{prop:nstep_variance}).
    }
    \captionvspace
    \label{fig:nstep_variance}
\end{figure}

\clearpage
\section{Proofs}
\label{app:proofs}

In this section, we include all proofs that were omitted from the main paper due to space constraints.

\subsection{Proof of \Cref{prop:nstep_variance}}
\label{app:prop_nstep_variance}

\propnstepvariance*
\begin{proof}
The $n$-step error can be expressed as a finite summation of TD errors:
\begin{equation}
    \label{eq:nstep_error}
    \nstep{n}_t - v(S_t)
    = \sum_{i=0}^{n-1} \gamma^i \delta_{t+i}
    \,.
\end{equation}
Using this, we calculate the covariance between two $n$-step returns with lengths $n_1$, $n_2$:
\begin{align}
    \nonumber
    \Cov[\nstep{n_1}_t,\nstep{n_2}_t \mid S_t]
    &= \Cov\Bigg[\sum_{i=0}^{n_1-1} \gamma^i \delta_{t+i}, \sum_{j=0}^{n_2-1} \gamma^j \delta_{t+j} \Biggm| S_t\Bigg] \\
    \nonumber
    &= \sum_{i=0}^{n_1-1} \sum_{j=0}^{n_2-1} \Cov[\gamma^i \delta_{t+i}, \gamma^j \delta_{t+j} \mid S_t] \\
    \nonumber
    &= \sum_{i=0}^{n_1-1} \sum_{j=0}^{n_2-1} \gamma^{i+j} \Cov[\delta_{t+i}, \delta_{t+j} \mid S_t] \\
    \nonumber
    &= \sum_{i=0}^{n_1-1} \sum_{j=0}^{n_2-1} \gamma^{i+j} ((1-\rho) \vone_{i=j} + \rho) \kappa \\
    \nonumber
    &= (1-\rho) \smashoperator{\sum_{i=0}^{\min(n_1,n_2)-1}} \gamma^{2i} \kappa + \rho \sum_{i=0}^{n_1-1} \sum_{j=0}^{n_2-1} \gamma^{i+j} \kappa \\
    \nonumber
    &= (1-\rho) \smashoperator{\sum_{i=0}^{\min(n_1,n_2)-1}} \gamma^{2i} \kappa + \rho \sum_{i=0}^{n_1-1} \gamma^i \sum_{j=0}^{n_2-1} \gamma^j \kappa \\
    \label{eq:nstep_covariance}
    &= (1-\rho) \gammafunc{2}{\min(n_1,n_2)} \kappa + \rho \gammafunc{1}{n_1} \gammafunc{1}{n_2} \kappa
    \,.
\end{align}
Because $\Var[\nstep{n}_t \mid S_t] = \Cov[\nstep{n}_t,\nstep{n}_t \mid S_t]$,
then by letting $n_1 = n_2 = n$, we obtain the $n$-step variance formula and the proof is complete.
\end{proof}

\subsection{Proof of \Cref{lemma:compound_error}}
\label{app:lemma_compound_error}

\lemmacompounderror*
\begin{proof}
We decompose the compound error into a weighted average of $n$-step errors, and then decompose those $n$-step errors into weighted sums of TD errors using \Cref{eq:nstep_error}:
\begin{align*}
    G^\vc_t - v(S_t)
    &= \left(\sum_{n=1}^\infty c_n \nstep{n}_t \right) - v(S_t) \\
    &= \sum_{n=1}^\infty c_n \left(\nstep{n}_t - v(S_t)\right) \\
    &= \sum_{n=1}^\infty c_n \sum_{i=0}^{n-1} \gamma^i \delta_{t+i} \\
    &= c_1 \delta_t + c_2 (\delta_t + \gamma \delta_{t+1}) + c_3 (\delta_t + \gamma \delta_{t+1} + \gamma^2 \delta_{t+2}) + \dots \\
    &= (c_1 + c_2 + \dots) \delta_t + \gamma (c_2 + c_3 + \dots) \delta_{t+1} + \gamma^2 (c_3 + c_4 + \dots) \delta_{t+2} + \dots \\
    &= h_0 \delta_t + \gamma h_1 \delta_{t+1} + \gamma^2 h_2 \delta_{t+2} + \dots \\
    &= \sum_{i=0}^\infty \gamma^i h_i \delta_{t+i}
    \,,
\end{align*}
which completes the lemma.
\end{proof}

\subsection{Proof of \Cref{prop:compound_variance}}
\label{app:prop_compound_variance}

\propcompoundvariance*
\begin{proof}
From \Cref{lemma:compound_error}, the variance of the compound return is
\begin{align*}
    \Var\left[G^\vc_t \mid S_t\right]
    &= \sum_{i=0}^\infty \sum_{j=0}^\infty \Cov[\gamma^i h_i \delta_{t+i}, \gamma^j h_j\delta_{t+j} \mid S_t] \\
    &= \sum_{i=0}^\infty \sum_{j=0}^\infty \gamma^{i+j} h_i h_j \Cov[\delta_{t+i}, \delta_{t+j} \mid S_t] \\
    &= \sum_{i=0}^\infty \sum_{j=0}^\infty \gamma^{i+j} h_i h_j ((1-\rho) \vone_{i=j} + \rho) \kappa \\
    &= (1-\rho) \sum_{i=0}^\infty \gamma^{2i} h_i^2 \kappa + \rho \sum_{i=0}^\infty \sum_{j=0}^\infty \gamma^{i+j} h_i h_j \kappa
    \,,
\end{align*}
which establishes \Cref{eq:compound_variance}.
\end{proof}

\subsection{\texorpdfstring{$\lambda$}{λ}-return Variance}
\label{app:lambda_variance}

We calculate the variance of the $\lambda$-return under our assumptions.
In the main text, we showed that the $\lambda$-return assigns a cumulative weight of $h_i = \lambda^i$ to the TD error at time $t+i$, which is also known from the TD($\lambda$) algorithm.
We can therefore apply \Cref{prop:compound_variance} to obtain the following variance expression:
\begin{align}
    \nonumber
    \Var[G^\lambda_t \mid S_t]
    &= (1-\rho) \sum_{i=0}^\infty (\gamma\lambda)^{2i} \kappa + \rho \sum_{i=0}^\infty \sum_{j=0}^\infty (\gamma\lambda)^{i+j} \kappa \\
    \nonumber
    &= (1-\rho) \sum_{i=0}^\infty (\gamma\lambda)^{2i} \kappa + \rho \sum_{i=0}^\infty (\gamma\lambda)^i \sum_{j=0}^\infty (\gamma\lambda)^j \kappa \\
    \label{eq:lambda_variance}
    &= \frac{(1-\rho) \kappa}{1 - (\gamma\lambda)^2} + \frac{\rho\kappa}{(1-\gamma\lambda)^2}
    \,.
\end{align}

\subsection{Proof of \Cref{prop:effective_nstep}}
\label{app:prop_effective_nstep}

\propeffectivenstep*
\begin{proof}
When $\gamma < 1$, we can take the logarithm of both sides of $\gamma^n = \sum_{k=1}^\infty c_k \gamma^k$ to get
\begin{equation*}
    n
    = \log_\gamma \! \left( \sum_{k=1}^\infty c_k \gamma^k \right)
    = \frac{\log \left( \sum_{k=1}^\infty c_k \gamma^k \right)}{\log \gamma}
    \,.
\end{equation*}
For the undiscounted case, we would like to evaluate this expression at $\gamma=1$;
however, since $\sum_{k=1}^\infty c_k = 1$ from the definition of a compound return, we arrive at an indeterminate form, $\smash{\frac{0}{0}}$.
Instead, we can apply L'H{\^o}pital's rule to evaluate the limit as $\gamma \to 1$:
\begin{align*}
    \lim_{\gamma \to 1} \frac{\log \left( \sum_{k=1}^\infty c_k \gamma^k \right)}{\log \gamma}
    &= \lim_{\gamma \to 1} \frac{\frac{d}{d\gamma} \log \left( \sum_{k=1}^\infty c_k \gamma^k \right)}{\frac{d}{d\gamma} \log \gamma} \\
    &= \lim_{\gamma \to 1} \frac{(\sum_{k=1}^\infty c_k \gamma^{k-1} k) \mathbin{/} (\sum_{k=1}^\infty c_k \gamma^k)}{1 \mathbin{/} \gamma} \\
    &= \lim_{\gamma \to 1} \frac{\sum_{k=1}^\infty c_k \gamma^k k}{\sum_{k=1}^\infty c_k \gamma^k} \\
    &= \frac{\sum_{k=1}^\infty c_k k}{\sum_{k=1}^\infty c_k} \\
    &= \sum_{k=1}^\infty c_k k
    \,,
\end{align*}
where the last step follows again from the fact that $\sum_{k=1}^\infty c_k = 1$.
This establishes the case where $\gamma = 1$ and completes the proof.
\end{proof}

\subsection{Proof of Theorem~\ref{theorem:variance_reduction}}
\label{app:theorem_variance_reduction}

\theoremvariancereduction*
\begin{proof}
\Cref{eq:nstep_covariance} gives us an expression for the covariance between two $n$-step returns.
We use this to derive an alternative formula for the variance of a compound return:
\begin{align*}
    \Var[G^\vc_t \mid S_t]
    &= \sum_{i=1}^\infty \sum_{j=1}^\infty \Cov[c_i \nstep{i}_t, c_j \nstep{j}_t \mid S_t] \\
    &= \sum_{i=1}^\infty \sum_{j=1}^\infty c_i c_j \Cov[\nstep{i}_t, \nstep{j}_t \mid S_t] \\
    &= \sum_{i=1}^\infty \sum_{j=1}^\infty c_i c_j \big((1-\rho) \gammafunc{2}{\min(i,j)}) + \rho \gammafunc{1}{i} \gammafunc{1}{j} \! \big) \kappa \\
    &= (1-\rho) \sum_{i=1}^\infty \sum_{j=1}^\infty c_i c_j \gammafunc{2}{\min(i,j)} \kappa + \rho \sum_{i=1}^\infty \sum_{j=1}^\infty c_i c_j \gammafunc{1}{i} \gammafunc{1}{j} \kappa
    \,.
\end{align*}
We analyze both sums separately, starting with the first term.
If we assume $\gamma < 1$ for now, then
$\gamma^n = \sum_{i=1}^\infty c_i \gamma^i$
by \Cref{prop:effective_nstep}.
Further note that
$\min(i,j) \leq (i + j) \mathbin{/} 2$,
with the inequality strict if $i \neq j$.
Because $\Gamma_{\!2}$ is monotonically increasing, it follows that
\begin{align*}
    \sum_{i=1}^\infty \sum_{j=1}^\infty c_i c_j \gammafunc{2}{\minfunc{i,j}}
    &< \sum_{i=1}^\infty \sum_{j=1}^\infty c_i c_j \gammafunc{2}{\frac{i+j}{2}} \\
    &= \sum_{i=1}^\infty \sum_{j=1}^\infty c_i c_j \left(\frac{1-\gamma^{i+j}}{1-\gamma^2}\right) \\
    &= \frac{1 - \sum_{i=1}^\infty \sum_{j=1}^\infty c_i c_j \gamma^{i+j}}{1-\gamma^2} \\
    &= \frac{1 - \sum_{i=1}^\infty c_i \gamma^i \sum_{j=1}^\infty c_j \gamma^j}{1-\gamma^2} \\
    &= \frac{1 - \gamma^{2n}}{1-\gamma^2} \\
    &= \gammafunc{2}{n}
    \,.
\end{align*}
The inequality is strict because at least two weights in $\vc$ are nonzero by definition of a compound return, guaranteeing at least one element in the sum has $i \neq j$.
If instead $\gamma = 1$, then ${n = \sum_{i=1}^\infty c_i i}$ by \Cref{prop:effective_nstep} and also $\gammafunc{2}{\min(i,j)} = \min(i,j)$.
Therefore, by Jensen's inequality, we have
\begin{align*}
    \sum_{i=1}^\infty \sum_{j=1}^\infty c_i c_j \gammafunc{2}{\minfunc{i,j}}
    &= \sum_{i=1}^\infty \sum_{j=1}^\infty c_i c_j \min(i,j) \\
    &< \minfunc{\sum_{i=1}^\infty c_i i,~\sum_{j=1}^\infty c_j j} \\
    &= \minfunc{n,~n} \\
    &= \gammafunc{2}{n}
    \,.
\end{align*}
Again, the inequality is strict by definition of a compound return, so we conclude that
\begin{equation*}
    \sum_{i=1}^\infty \sum_{j=1}^\infty c_i c_j \gammafunc{2}{\minfunc{i,j}}
    < \gammafunc{2}{n}
    ,\quad
    \text{for } 0 < \gamma \leq 1
    \,.
\end{equation*}

% Intentional new paragraph
We now address the second term.
We show that $\Gamma_{\!1}$ is invariant under a weighted average under our assumption that \Cref{eq:effective_nstep} holds.
If $\gamma < 1$, then
\begin{equation}
    \label{eq:Gamma_invariant}
    \sum_{i=1}^\infty c_i \gammafunc{1}{i}
    = \sum_{i=1}^\infty c_i \left(\frac{1-\gamma^i}{1-\gamma}\right)
    = \frac{1 - \sum_{i=1}^\infty c_i \gamma^i}{1-\gamma}
    = \frac{1 - \gamma^n}{1-\gamma}
    = \gammafunc{1}{n}
    \,.
\end{equation}
If $\gamma = 1$, then
\begin{equation*}
    \sum_{i=1}^\infty c_i \gammafunc{1}{i}
    = \sum_{i=1}^\infty c_i i
    = n
    = \gammafunc{1}{n}
    \,.
\end{equation*}
Thus, regardless of $\gamma$, the second term becomes
\begin{equation*}
    \sum_{i=1}^\infty \sum_{j=1}^\infty c_i c_j \gammafunc{1}{i} \gammafunc{1}{j}
    = \sum_{i=1}^\infty c_i \gammafunc{1}{i} \sum_{j=1}^\infty c_j \gammafunc{1}{j}
    = \gammafunc{1}{n}^2
    \,.
\end{equation*}

% Intentional new paragraph
Putting everything together, we have so far shown that
\begin{equation*}
    \Var[G^\vc_t \mid S_t]
    \leq (1-\rho) \gammafunc{2}{n} \kappa + \rho \gammafunc{1}{n}^2 \kappa
    \,,
\end{equation*}
where the right-hand side is the $n$-step return variance given by \Cref{prop:nstep_variance}.
As we showed above, this inequality is strict whenever the first term is active, i.e., $\rho < 1$, which completes the proof.
\end{proof}

\subsection{Proof of \Cref{coro:lambda_variance_reduction}}
\label{app:lambda_variance_reduction}

\corolambdavariancereduction*

\begin{proof}
First, the left-hand side of the inequality follows immediately from \Cref{theorem:variance_reduction}.

For the right-hand side, we next make the observation that, for any compound return,
\begin{align*}
    \sum_{i=0}^\infty \gamma^i h_i
    = \sum_{i=0}^\infty \sum_{n=i+1}^\infty \gamma^i c_n
    = \sum_{n=1}^\infty \sum_{i=0}^{n-1} \gamma^i c_n
    = \sum_{n=1}^\infty c_n \frac{1-\gamma^n}{1-\gamma}
    = \frac{1 - \sum_{n=1}^\infty c_n \gamma^n}{1-\gamma}
    = \frac{1-\beta}{1-\gamma}
    \,.
\end{align*}
This allows us to rewrite the variance model in \Cref{prop:compound_variance} as
\begin{equation*}
    \Var[G^\vc_t \mid S_t]
    = (1-\rho) \sum_{i=0}^\infty \gamma^{2i} h_i^2 \kappa + \rho \sum_{i=0}^\infty \sum_{j=0}^\infty \gamma^{i+j} h_i h_j \kappa
    = (1-\rho) \sum_{i=0}^\infty \gamma^{2i} h_i^2 \kappa + \rho \left(\frac{1-\beta}{1-\gamma}\right)^2 \kappa
    \,.
\end{equation*}
The second term is identical for all returns with a fixed contraction modulus $\beta$, as is assumed in \Cref{theorem:variance_reduction}.
Thus, by subtracting the variance of both returns, we get the following expression for variance reduction (positive means more variance reduction):
\begin{equation*}
    \Var[G^n_t \mid S_t] - \Var[G^\vc_t \mid S_t]
    = (1-\rho) \left(\Gamma_2(n) - \sum_{i=0}^\infty \gamma^{2i} h_i^2\right) \kappa
    \,.
\end{equation*}
This reveals that the variance reduction of a compound return is proportional to $1-\rho$, affirming the result of \Cref{theorem:variance_reduction} that variance reduction occurs whenever $\rho < 1$.
Additionally, this variance reduction is proportional to the critical quantity
$\Gamma_2(n) - \sum_{i=0}^\infty \gamma^{2i} h_i^2$,
a function of the weights of the compound return.
Substituting the weights for a $\lambda$-return, this critical quantity becomes
\begin{equation*}
    \Gamma_2(n) - \frac{1}{1-(\gamma \lambda)^2}
    \,.
\end{equation*}
This gap is monotonically increasing with $\gamma$ and attains its maximum value as $\gamma \to 1$, i.e., the maximum-variance case of undiscounted rewards.\footnote{
    To see this, make the substitution $\lambda = (1 - \gamma^{n-1}) \mathbin{/} (1 - \gamma^n)$ from \Cref{prop:lambda_effective_nstep} because the effective $n$-step is the same, and analyze the behavior as $\gamma \to 1$.
}
Taking the limit as $\gamma \to 1$ makes the possible variance reduction proportional to
\begin{equation*}
    n - \frac{1}{1-\lambda^2}
    = \frac{1}{1-\lambda} - \frac{1}{1-\lambda^2}
    = \frac{\lambda}{1-\lambda^2}
    \,,
\end{equation*}
where the identity $n = 1 \mathbin{/} (1-\lambda)$ comes from \Cref{prop:lambda_effective_nstep},
the effective $n$-step of the $\lambda$-return.
This gives us the final upper bound:
\begin{equation*}
    \Var[G^n_t \mid S_t] - \Var[G^\lambda_t \mid S_t]
    \leq (1-\rho) \frac{\lambda}{1-\lambda^2} \kappa
    \,,
\end{equation*}
which completes the inequality.
\end{proof}

\subsection{Finite-Time Analysis}
\label{app:ft_analysis}

In this section, we prove \Cref{theorem:ft_analysis} to establish a finite-time bound on the performance of multistep TD learning.
We derive the bound in terms of the return's contraction modulus and variance, allowing us to invoke \Cref{theorem:variance_reduction} and show an improved convergence rate.

At each iteration, TD learning updates the current parameters $\theta_t \in \R^d$ according to \Cref{eq:td_lfa}.
A value-function estimate for any state $s$ is obtained by evaluating the dot product of the parameters and the state's corresponding feature vector:
$\smash{v_\theta(s) \defeq \theta\tran \phi(s)}$.
Following \citet{bhandari2018finite},
we assume that
$\smash{\sqnorm{\phi(s)}} \leq 1$,
$\forall~s \in \mathcal{S}$.
This can be guaranteed in practice by normalizing the features and is therefore not a strong assumption.

In the prediction setting, the agent's behavior policy is fixed such that the MDP can be cast as a Markov reward process (MRP), where $r(s,s')$ denotes the expected reward earned when transitioning from state $s$ to state $s'$.
We adopt the i.i.d.\ state model from \citet[][Sec.~3]{bhandari2018finite} and generalize it for multistep TD updates.
\begin{assumption}[i.i.d.\ state model]
    \label{assump:iid}
    Assume the MRP under the policy $\pi$ is ergodic.
    Let $\smash{\vd \in \R^{|\mathcal{S}|}}$ represent the MRP's unique stationary distribution.
    Each iteration of \Cref{eq:td_lfa} is calculated by first sampling a random initial state $S_{t,0} \sim \vd$ and then sampling a trajectory of subsequent states
    $S_{t,i+1} \sim \Pr(\cdot \mid S_{t,i})$,
    $\forall~i \geq 0$.
\end{assumption}
That is, a state $S_{t,0}$ sampled from the steady-state MRP forms the root node for the following trajectory
$(S_{t,1}, S_{t,2}, \dots)$
that is generated according to the MRP transition function.
Notably, this setting parallels the experience-replay setting utilized by many deep RL agents.

To facilitate our analysis, we decompose the compound TD updates into weighted averages of $n$-step TD updates, where each $n$-step update has the form
\begin{equation*}
    g^n_t(\theta) \defeq \left(\nstep{n}_t - \mathop{\phi(S_t)\tran} \theta\right) \mathop{\phi(S_t)}
    \,.
\end{equation*}
This allows us to conveniently express a compound TD update as
\begin{equation*}
    g^\vc_t(\theta) \defeq \sum_{n=1}^\infty c_n g^n_t(\theta)
    \,.
\end{equation*}
Our proofs also make use of the \emph{expected} $n$-step TD update:
\begin{equation*}
    \bar{g}^n(\theta)
    \defeq \sum_{s_0 \in \mathcal{S}} \sum_{\tau \in \mathcal{S}^{n-1}} d(s_0) \Pr(\tau \mid s_0) \left(\nstep{n}(s_0,\tau,\theta) - \phi(s_0)\tran \theta \right) \phi(s_0)
\end{equation*}
where
$(s_1,s_2,\dots) = \tau$
and
$\nstep{n}(s_0,\tau,\theta) \defeq r(s_0,s_1) + \dots + \gamma^{n-1} r(s_{n-1},s_n) + \gamma^n \phi(s_n)\tran \theta$
is the $n$-step return generated from $(s_0,\tau)$.

For brevity, let
$\smash{R_i \defeq r(S_{t,i}\,, S_{t,i+1})}$
and
$\smash{\phi_i \defeq \phi(S_{t,i})}$
be random variables sampled according to \Cref{assump:iid}.
We more conveniently write the expected $n$-step TD update as
\begin{equation}
    \label{eq:expected_compound_update}
    \bar{g}^n(\theta)
    = \expect{\phi_0 (R_0 + \gamma R_1 + \dots + \gamma^{n-1} R_{n-1})} + \expect{\phi_0 (\gamma^n \phi_n - \phi_0)\tran} \theta
    \,.
\end{equation}
The expected compound TD update easily follows as the weighted average
\begin{equation*}
    \bar{g}^\vc(\theta)
    \defeq \sum_{n=1}^\infty c_n \bar{g}^n(\theta)
    \,.
\end{equation*}
Finally, let $\theta^*$ be the fixed point of the compound TD update:
i.e., $\bar{g}^\vc(\theta^*) = 0$.
This fixed point always exists and is unique because the projected Bellman operator is a contraction mapping \citep{tsitsiklis1997analysis}, and therefore so is any weighted average of the $n$-iterated operators.

Before we prove \Cref{theorem:ft_analysis}, we must introduce two lemmas.
The first establishes a lower bound on the angle between the expected TD update and the true direction toward the fixed point.
\begin{lemma}
    \label{lemma:dot_lower_bound}
    Define the diagonal matrix
    $\smash{\mD \defeq \diag(\vd)}$.
    For any $\theta \in \mathbb{R}^d$,
    \begin{equation}
        (\theta^* - \theta)\tran \bar{g}^\vc(\theta)
        \geq (1 - \beta) \norm{v_{\theta^*} - v_\theta}_\mD^2
        \,.
    \end{equation}
\end{lemma}
\begin{proof}
    Let $\xi_i \defeq v_{\theta^*}(S_{t,i}) - v_\theta(S_{t,i}) = (\theta^* - \theta)\tran \phi_i$ for $i \geq 0$.
    By stationarity, each $\xi_i$ is a correlated random variable with the same marginal distribution.
    Because $S_{t,0}$ is drawn from the stationary distribution, we have
    $\expect{\xi_i^2} = \norm{v_{\theta^*} - v_\theta}_\mD^2$.

    From \Cref{eq:expected_compound_update}, we show
    \begin{align*}
        \bar{g}^\vc(\theta)
        = \bar{g}^\vc(\theta) - \bar{g}^\vc(\theta^*)
        &= \sum_{n=1}^\infty c_n \expect{\phi_0 (\gamma^n \phi_n - \phi_0)\tran (\theta - \theta^*)} \\
        &= \sum_{n=1}^\infty c_n \expect{\phi_0 (\xi_0 - \gamma^n \xi_n)}
        \,.
    \end{align*}
    It follows that
    \begin{align*}
        (\theta^* - \theta)\tran \bar{g}^\vc(\theta)
        &= \sum_{n=1}^\infty c_n \expect{\xi_0 (\xi_0 - \gamma^n \xi_n)} \\
        &= \expect{\xi_0^2} - \sum_{n=1}^\infty c_n \gamma^n \expect{\xi_0 \xi_n} \\
        &\geq \left(1 - \sum_{n=1}^\infty c_n \gamma^n\right) \expect{\xi_0^2} \\
        &= (1 - \beta) \norm{v_{\theta^*} - v_\theta}_\mD^2
        \,.
    \end{align*}
    The inequality uses the Cauchy-Schwarz inequality along with the fact that every $\xi_i$ has the same marginal distribution:
    thus, $\expect{\xi_0 \xi_i} \leq \sqrt{\expect{\xi_0^2}} \sqrt{\expect{\xi_i^2}} = \expect{\xi_0^2}$.
\end{proof}
The next lemma establishes a bound on the second moment of the squared norm of the TD update in terms of the contraction modulus $\beta$ and the variance $\sigma^2$ of the compound return.
\begin{lemma}
    \label{lemma:grad_2nd_moment}
    Define $\Delta^* \defeq \norm{r}_\infty + (1+\gamma) \norm{\theta^*}_\infty$ and $C \defeq \Delta^* \mathbin{/} (1-\gamma)$.
    For any $\theta \in \mathbb{R}^d$,
    \begin{equation*}
        \E[\sqnorm{g_t(\theta)}]
        \leq 2 (1-\beta)^2 C^2 + 2 \sigma^2 + 4 (1 + \beta) \norm{v_{\theta^*} - v_\theta}_\mD^2
        \,.
    \end{equation*}
\end{lemma}
\begin{proof}
    Let
    $\delta^*_{t,i} \defeq R_i + \gamma \phi_{i+1}\tran \theta^* - \phi_i\tran \theta^*$
    and note that
    $\abs{\delta^*_{t,i}} \leq \Delta^*$
    for all
    $i \geq 0$
    by the triangle inequality and the bounded-feature assumption.
    Denote the $n$-step and compound errors constructed from $\theta^*$ by
    $\delta^{(n)}_t \defeq \sum_{i=0}^{n-1} \gamma^i \delta^*_{t,i}$
    and
    $\delta^\vc_t \defeq \sum_{n=1}^\infty c_n \delta^{(n)}_t$,
    respectively.
    We have
    \begin{equation}
        \label{eq:2nd_moment_bound_init}
        \expect{\sqnorm{g_t(\theta^*)}}
        = \expect{\sqnorm{\delta_t^\vc \phi_0}}
        \leq \expect{(\delta_t^\vc)^2}
        = \expect{\delta^\vc_t}^2 + \sigma^2
        \,,
    \end{equation}
    where the inequality follows from the assumption that $\sqnorm{\phi_0} \leq 1$.
    The absolute value of the expectation can be bounded using the triangle inequality:
    \begin{align}
        \label{eq:abs_expected_compound_error}
        \Big| \expect{\delta^\vc_t} \! \Big|
        = \abs{\expect{\sum_{n=1}^\infty c_n \delta_t^{(n)}}}
        \leq \sum_{n=1}^\infty c_n \gammafunc{1}{n} \Delta^*
        = \frac{1-\beta}{1-\gamma} \Delta^*
        = (1-\beta) C
        \,.
    \end{align}
    The identity
    $\sum_{n=1}^\infty c_n \Gamma_1(n) = (1-\beta) \mathbin{/} (1-\gamma)$
    comes from \Cref{eq:Gamma_invariant}.
    \Cref{eq:2nd_moment_bound_init,eq:abs_expected_compound_error} imply
    \begin{equation}
        \label{eq:2nd_moment_bound_final}
        \expect{\sqnorm{g_t(\theta^*)}}
        \leq (1-\beta)^2 C^2 + \sigma^2
        \,.
    \end{equation}
    Recall that
    $\expect{\xi_i^2} = \norm{v_{\theta^*} - v_\theta}_\mD^2$ for all $i \geq 0$.
    Next, we show
    \begin{align}
        \nonumber
        \expect{\sqnorm{g_t(\theta) - g_t(\theta^*)}}
        &= \expect{\sqnorm{\sum_{n=1}^\infty c_n \phi_0 (\gamma^n \phi_n - \phi_0)\tran (\theta - \theta^*)}} \\
        \nonumber
        &= \expect{\sqnorm{\sum_{n=1}^\infty c_n \phi_0 (\xi_0 - \gamma^n \xi_n)}} \\
        \nonumber
        &\leq \sum_{n=1}^\infty c_n \expect{\sqnorm{\phi_0 (\xi_0 - \gamma^n \xi_n)}} \\
        \nonumber
        &\leq \sum_{n=1}^\infty c_n \expect{(\xi_0 - \gamma^n \xi_n)^2} \\
        \nonumber
        &\leq 2 \sum_{n=1}^\infty c_n \left( \expect{\xi_0^2} + \gamma^{2n} \expect{\xi_n^2} \right) \\
        \nonumber
        &= 2 \sum_{n=1}^\infty c_n (1 + \gamma^{2n}) \norm{v_{\theta^*} - v_\theta}_\mD^2 \\
        \nonumber
        &\leq 2 \sum_{n=1}^\infty c_n (1 + \gamma^n) \norm{v_{\theta^*} - v_\theta}_\mD^2 \\
        \label{eq:diff_2nd_moment_bound}
        &= 2 (1 + \beta) \norm{v_{\theta^*} - v_\theta}_\mD^2
        \,.
    \end{align}
    The four inequalities respectively follow from Jensen's inequality, the bounded-feature assumption
    $\sqnorm{\phi} \leq 1$,
    the triangle inequality,
    and the fact that $\gamma^{2n} \leq \gamma^n$.
    The final equality comes from the definition of the contraction modulus, the coefficient of the left-hand side of \Cref{eq:compound_erp}.
    Combining \Cref{eq:2nd_moment_bound_final,eq:diff_2nd_moment_bound} gives the final result:
    \begin{align*}
        \expect{\sqnorm{g_t(\theta)}}
        &\leq \expect{\left(\norm{g_t(\theta^*)}_2 + \norm{g_t(\theta) - g_t(\theta^*)}_2\right)^2} \\
        &\leq 2 \expect{\sqnorm{g_t(\theta^*)}} + 2 \expect{\sqnorm{g_t(\theta) - g_t(\theta^*)}} \\
        &\leq 2 (1-\beta)^2 C^2 + 2 \sigma^2 + 4 (1+\beta) \norm{v_{\theta^*} - v_\theta}_\mD^2
        \,,
    \end{align*}
    where the second inequality uses the algebraic identity
    $(x + y)^2 \leq 2x^2 + 2y^2$.
\end{proof}

We are now ready to derive the finite-time bound.
We restate \Cref{theorem:ft_analysis} and then provide the proof.

\theoremftanalysis*

\begin{proof}
TD learning updates the parameters according to \Cref{eq:td_lfa}.
Therefore,
\begin{align*}
    \sqnorm{\theta^* - \theta_{t+1}}
    &= \sqnorm{\theta^* - \theta_t - \alpha \mathop{g_t(\theta_t)}} \\
    &= \sqnorm{\theta^* - \theta_t} - 2 \alpha \mathop{g_t(\theta_t)\tran} (\theta^* - \theta_t) + \alpha^2 \sqnorm{g_t(\theta)}
    \,.
\end{align*}
Taking the expectation and then applying \Cref{lemma:dot_lower_bound,lemma:grad_2nd_moment} gives
\begin{align*}
    \expect{\sqnorm{\theta^* - \theta_{t+1}}}
    &= \expect{\sqnorm{\theta^* - \theta_t}} - 2 \alpha \expect{\mathop{g_t(\theta_t)\tran} (\theta^* - \theta_t)} + \alpha^2 \expect{\sqnorm{g_t(\theta)}} \\
    &= \expect{\sqnorm{\theta^* - \theta_t}} - 2 \alpha \expect{\expect{\mathop{g_t(\theta_t)\tran} (\theta^* - \theta_t)} \mid \theta_t} + \alpha^2 \expect{\expect{\sqnorm{g_t(\theta)}} \mid \theta_t} \\
    \nonumber
    &\leq \expect{\sqnorm{\theta^* - \theta_t}} - \left(2 \alpha (1-\beta) - 4 \alpha^2 (1+\beta)\right) \norm{v_{\theta^*} - v_\theta}_\mD^2 + 2 \alpha^2 \left((1-\beta)^2 C^2 + \sigma^2\right) \\
    &\leq \expect{\sqnorm{\theta^* - \theta_t}} - \alpha (1-\beta) \norm{v_{\theta^*} - v_\theta}_\mD^2 + 2 \alpha^2 \left((1-\beta)^2 C^2 + \sigma^2\right)
    \,.
\end{align*}
The first inequality is due to \Cref{lemma:dot_lower_bound,lemma:grad_2nd_moment}, which are applicable due to the i.i.d.\ setting (because the trajectory influencing $g_t$ is independent of $\theta_t$).
The second inequality follows from the assumption that $\alpha \leq (1-\beta) \mathbin{/} 4$.
Rearranging the above inequality gives us
\begin{align*}
    \expect{\norm{v_{\theta^*} - v_{\theta_t}}_\mD^2}
    \leq \frac{\sqnorm{\theta^* - \theta_t} - \sqnorm{\theta^* - \theta_{t+1}} + 2 \alpha^2 \left((1-\beta)^2 C^2 + \sigma^2\right)}{\alpha (1-\beta)}
    \,.
\end{align*}
Summing over $T$ iterations and then invoking the assumption that $\alpha = 1 \mathbin{/} \sqrt{T}$:
\begin{align*}
    \sum_{t=0}^{T-1} \expect{\norm{v_{\theta^*} - v_{\theta_t}}_\mD^2}
    &\leq \frac{\sqnorm{\theta^* - \theta_0} - \sqnorm{\theta^* - \theta_T}  + 2 \alpha^2 \left((1-\beta)^2 C^2 + \sigma^2\right) T}{\alpha (1-\beta)} \\
    &\leq \frac{\sqnorm{\theta^* - \theta_0} + 2 \alpha^2 \left((1-\beta)^2 C^2 + \sigma^2\right) T}{\alpha (1-\beta)} \\
    &= \frac{\sqnorm{\theta^* - \theta_0} \sqrt{T} + 2 \left((1-\beta)^2 C^2 + \sigma^2\right) \sqrt{T}}{1-\beta}
    \,.
\end{align*}
We therefore conclude that
\begin{equation*}
    \expect{\norm{v_{\theta^*} - v_{\bar{\theta}_T}}_\mD^2}
    \leq \frac{1}{T} \sum_{t=0}^{T-1} \expect{\norm{v_{\theta^*} - v_{\theta_t}}_\mD^2}
    \leq \frac{\sqnorm{\theta^* - \theta_0} + 2 (1-\beta)^2 C^2 + 2 \sigma^2}{(1-\beta) \sqrt{T}}
    \,,
\end{equation*}
which completes the bound.
\end{proof}

\subsection{Proof of \Cref{prop:td_solution_quality}}
\label{app:prop_td_solution_quality}

\proptdsolutionquality*

\begin{proof}
Our proof generalizes the bound for the $\lambda$-return operator given by \citet[][Lemma~6]{tsitsiklis1997analysis}.
Let $T_\pi \colon v \mapsto r + \gamma \mP_\pi v$ be the Bellman operator, where $\mP_\pi$ is the stochastic transition matrix of the MDP under policy $\pi$.
The $n$-step Bellman operator is denoted by $T_\pi^n$, where $T_\pi^n v \defeq T_\pi (T_\pi^{n-1} v)$ and $T_\pi^0 v \defeq v$.

A compound return corresponds to the operator $T_\pi^{(\vc)} \colon v \mapsto \sum_{n=1}^\infty c_n T_\pi^n v$.
This operator is a contraction mapping since it is a convex combination of contraction mappings.
Let $\beta$ be the contraction modulus of $\smash{T_\pi^{(\vc)}}$ for the given weights, $\vc$.
The compound TD procedure in \Cref{eq:backup} converges to a fixed point $\theta^*$ which is the unique solution of the following projected Bellman equation:
\begin{equation*}
    \mPi T_\pi^{(\vc)} (\mPhi \theta^*) = \mPhi \theta^*
    \,,
\end{equation*}
where $\mPi = \mPhi (\mPhi\tran \mD \mPhi)^{-1} \mPhi\tran \mD$ is the linear projection operator.
In their proof of Lemma~6, \citeauthor{tsitsiklis1997analysis} show that $\mPi$ is nonexpansive with respect to the norm $\norm{\cdot}_\mD$.
Therefore, for any compound return, it follows that
\begin{align*}
    \norm{\mPhi \theta^* - v_\pi}_\mD
    &\leq \norm{\mPhi \theta^* - \mPi v_\pi}_\mD + \norm{\mPi v_\pi - v_\pi}_\mD \\ 
    &= \norm{\mPi T_\pi^{(\vc)} (\mPhi \theta^*) - \mPi v_\pi}_\mD + \norm{\mPi v_\pi - v_\pi}_\mD \\
    &\leq \norm{T_\pi^{(\vc)}(\mPhi \theta^*) - v_\pi}_\mD + \norm{\mPi v_\pi - v_\pi}_\mD \\
    &\leq \beta \norm{\mPhi \theta^* - v_\pi}_\mD + \norm{\mPi v_\pi - v_\pi}_\mD
\end{align*}
Solving the inequality for $\norm{\mPhi \theta^* - v_\pi}_\mD$ gives the final bound.
\end{proof}

\subsection{Proof of \Cref{prop:lambda_effective_nstep}}
\label{app:prop_lambda_effective_nstep}

\propeffectivelambda*
\begin{proof}
\textbf{Case $\gamma < 1$:}
We substitute
$c_n = (1-\lambda) \lambda^{n-1}$
into the weighted average in \Cref{eq:compound_erp} to compute the contraction modulus of the $\lambda$-return:
\begin{equation*}
    \sum_{n=1}^\infty c_n \gamma^n
    = \sum_{n=1}^\infty (1-\lambda) \lambda^{n-1} \gamma^{n}
    = \gamma (1-\lambda) \sum_{n=1}^\infty (\gamma\lambda)^{n-1}
    = \frac{\gamma (1-\lambda)}{1-\gamma\lambda}
    \,.
\end{equation*}
We therefore seek $\lambda$ such that
\begin{equation*}
    \frac{\gamma (1-\lambda)}{1-\gamma\lambda} = \gamma^n
\end{equation*}
in order to equate the $\lambda$-return's contraction modulus to that of the given $n$-step return.
We multiply both sides of the equation by $1-\gamma\lambda$ and isolate $\lambda$ to complete the case:
\begin{align*}
    \gamma (1-\lambda) &= \gamma^n (1-\gamma\lambda) \\
    1-\lambda &= \gamma^{n-1} (1-\gamma\lambda) \\
    1-\lambda &= \gamma^{n-1} - \gamma^n \lambda \\
    \gamma^n \lambda - \lambda &= \gamma^{n-1} - 1 \\
    \lambda(\gamma^n - 1) &= \gamma^{n-1} - 1 \\
    \lambda &= (1 - \gamma^{n-1}) \mathbin{/} (1 - \gamma^n)
    \,.
\end{align*}

% Intentional new paragraph
\textbf{Case $\gamma = 1$:}
We use \Cref{prop:effective_nstep} with
$c_n = (1-\lambda) \lambda^{n-1}$
to compute the effective $n$-step of the $\lambda$-return:
\begin{equation*}
    n
    = \sum_{k=1}^\infty (1-\lambda) \lambda^{k-1} k
    = (1-\lambda) \sum_{k=1}^\infty \lambda^{k-1} k
    = (1-\lambda) \frac{1}{(1-\lambda)^2}
    = \frac{1}{1-\lambda}
    \,.
\end{equation*}
Rearranging the equation $n = 1 \mathbin{/} (1-\lambda)$ for $\lambda$ gives the final result of $\lambda = (n-1) \mathbin{/} n$.
\end{proof}

\clearpage
\section{Pilar: Piecewise \texorpdfstring{$\lambda$}{λ}-Return}
\label{app:pilar}

\begin{wraptable}{r}{0.4\textwidth}
    \caption{$n$-step returns and Pilars with equal contraction moduli when $\gamma = 0.99$.}
    \label{tab:pilar}
    \begin{center}
    \begin{tabular}{l|rrr}
        \toprule
        effective\\$n$-step & $n_1$ & $n_2$ & $c$ \\
        \midrule
        2 & 1 & 4 & 0.337 \\
        3 & 1 & 6 & 0.406 \\
        4 & 2 & 7 & 0.406 \\
        5 & 2 & 9 & 0.437 \\
        10 & 4 & 16 & 0.515 \\
        20 & 6 & 35 & 0.519 \\
        25 & 8 & 43 & 0.530 \\
        50 & 13 & 79 & 0.640 \\
        100 & 22 & 147 & 0.760 \\
        \bottomrule
    \end{tabular}
    \end{center}
\end{wraptable}

We present a basic search algorithm for finding the corresponding Pilar for a given $n$-step return (see \Cref{algo:pilar}).
The algorithm accepts the desired effective $n$-step (which does not need to be an integer necessarily) as its only argument and returns the values $(n_1,n_2,c)$ such that the two-bootstrap return
${(1-c) \nstep{n_1}_t} + {c \nstep{n_2}_t}$
minimizes the maximum absolute difference between its cumulative weights and those of the $\lambda$-return with the same effective $n$-step.
The algorithm proceeds as follows.
For each $n_1 \in \{1, \dots, \lfloor n \rfloor\}$, scan through $n_2 \in \{n_1 + 1, n_1 + 2, \dots,\}$ until the error stops decreasing.
Every time a better $(n_1,n_2)$-pair is found, record the values, and return the last recorded values upon termination.
The resulting Pilar has the same contraction modulus as the targeted $n$-step return;
thus, their error-reduction properties are the same, but the Pilar's variance is lower by \Cref{theorem:variance_reduction}.
To modify the search algorithm for undiscounted rewards, we just need to change $\lambda$ and $c$ such that they equate the COMs---rather than the contraction moduli---of the two returns.
We also include this case in \Cref{algo:pilar}.

We populate \Cref{tab:pilar} with corresponding Pilar values for several common $n$-step returns when $\gamma=0.99$.
A discount factor of $\gamma=0.99$ is extremely common in deep RL, and so it is hoped that this table serves as a convenient reference that helps practitioners avoid redundant searches with \Cref{algo:pilar}.

\begin{minipage}{0.5\textwidth}
\begin{algorithm}[H]
\caption{Pilar($n$)}
\label{algo:pilar}
\begin{algorithmic}[1]
    \State {\bfseries require} $n \geq 1$,\enskip $\gamma \in (0,1)$
    \State $\lambda = \begin{cases}
        (1-\gamma^{n-1}) \mathbin{/} (1-\gamma^n) \text{ if } \gamma < 1 \\
        (n-1) \mathbin{/} n \text{ if } \gamma = 1
    \end{cases}$
    \State $\text{best\_error} \gets \infty$
    \For{$n_1 = 1, \dots, \lfloor n\rfloor$}
        \State $n_2 \gets \lfloor n\rfloor$
        \State $\text{error} \gets \infty$
        \Repeat
            \State $n_2 \gets n_2 + 1$
            \State $c = \begin{cases}
                (\gamma^n - \gamma^{n_1}) \mathbin{/} (\gamma^{n_2} - \gamma^{n_1}) \text{ if } \gamma < 1 \\
                (n - n_1) \mathbin{/} (n_2 - n_1) \text{ if } \gamma = 1
            \end{cases}$
            \State $\text{prev\_error} \gets \text{error}$
            \State $\text{error} \gets \Call{error}{\lambda,n_1,n_2,c}$
            \If{$\text{error} < \text{best\_error}$}
                \State $\text{values} \gets (n_1,n_2,c)$
                \State $\text{best\_error} \gets \text{error}$
            \EndIf
        \Until{$\text{error} \geq \text{prev\_error}$}
    \EndFor
    \State \Return $\text{values}$
    \Statex
    \Function{error}{$\lambda,n_1,n_2,c$}
        \State Let $h_i = \begin{cases}
            1 & \text{if } i < n_1 \\
            c & \text{else if } i < n_2 \\
            0 & \text{else} \\
        \end{cases}$
        \State \Return $\max_{i \geq 0} \abs{\gamma^i h_i - (\gamma\lambda)^i}$
    \EndFunction
\end{algorithmic}
\end{algorithm}
\end{minipage}

\clearpage
\section{Experiment Setup and Additional Results}
\label{app:experiment_setup}

Our DQN experiment procedure closely matches that of \citet{young2019minatar}.
The only differences in our methodology are the mean-squared loss (rather than Huber loss), the Adam optimizer \citep{kingma2015adam}, and multistep returns.
MinAtar represents states as $10 \times 10 \times 7$ binary images.
The agents process these with a convolutional network;
the first layer is a 16-filter $3 \times 3$ convolutional layer, the output of which is flattened and then followed by a dense layer with 128 units.
Both layers use ReLU activations.

The agents were trained for 5 million time steps each.
They executed a random policy for the first 5k time steps to prepopulate the replay buffer (capacity: 100k transitions), and then switched to an $\epsilon$-greedy policy for the remainder of training, with $\epsilon$ annealed linearly from $1$ to $0.1$ over the next 100k steps.
Every step, the main network was updated using a minibatch of 32 return estimates to minimize the loss in \Cref{eq:dqn_multistep}.
The target network's parameters were copied from the main network every 1k time steps.

To obtain the $n$-step returns, the replay buffer is modified to return a minibatch of sequences of $n+1$ experiences for each return estimate (instead of the usual two experiences for DQN).
The return is computed by summing the first $n$ rewards and then adding the value-function bootstrap from the final experience, with discounting if $\gamma < 1$.
If the episode terminates at any point within this trajectory, then the return is truncated and no bootstrapping is necessary, since the value of a terminal state is defined to be zero.
For Pilars, the idea is the same, but the trajectories must have length $n_2 + 1$ to accommodate the lengths of both $n$-step returns.
The two returns are computed as above, and then combined by averaging them:
$(1-c) \nstep{n_1} + c \nstep{n_2}$.

We show the learning curves for each return estimator using its best step size in \Cref{fig:lcurves_dqn_all}, where the step size was chosen from
$\{10^{-5}, 3 \times 10^{-5}, 10^{-4}, 3 \times 10^{-4}, 10^{-3}\}$.
We observe that Pilars perform no worse than $n$-step returns in the MinAtar domain, and sometimes significantly better.
The relative improvement between Pilars and $n$-step returns tends to widen for $n=5$, suggesting that Pilars are better at tolerating the higher variance in this setting.

In \Cref{fig:lcurves_ppo_all}, we also show similar learning curves for the PPO experiments.
Other than the return estimators themselves, we use the default hyperparameters from the CleanRL implementation, which are similar to those of \citet{schulman2017proximal}.
The network architectures for the actor and critic are 2-layer, 64-unit dense networks with tanh activations.

\begin{figure}[ht]
    \centering
    \includegraphics[width=0.24\textwidth]{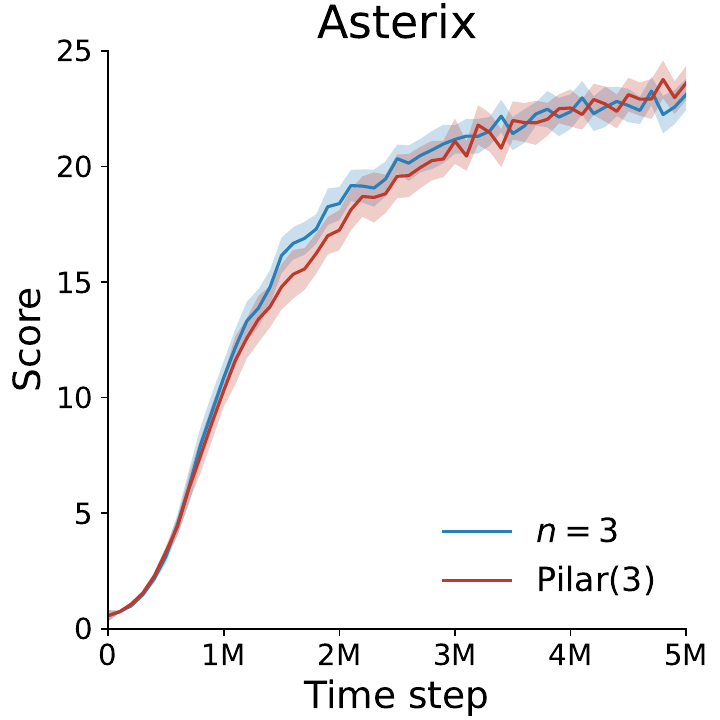}
    \hfill
    \includegraphics[width=0.24\textwidth]{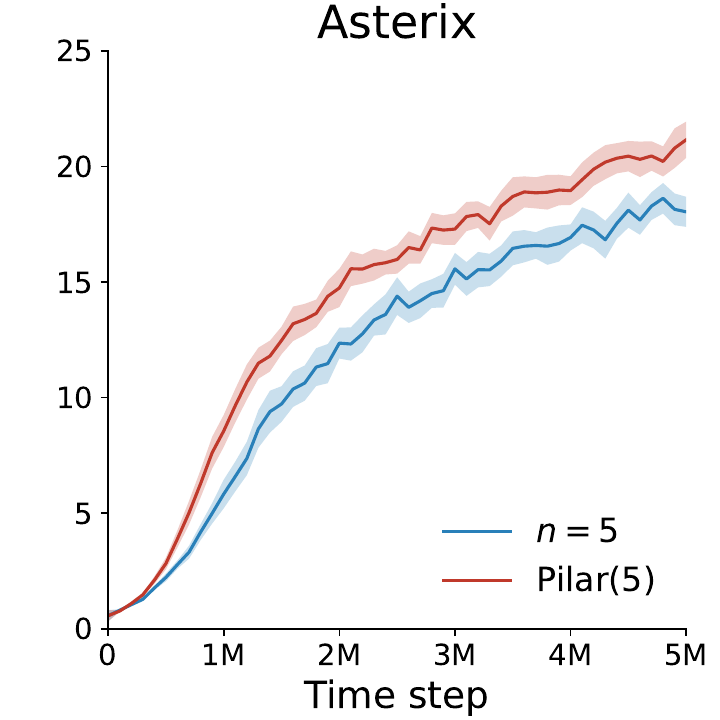}
    \hfill
    \includegraphics[width=0.24\textwidth]{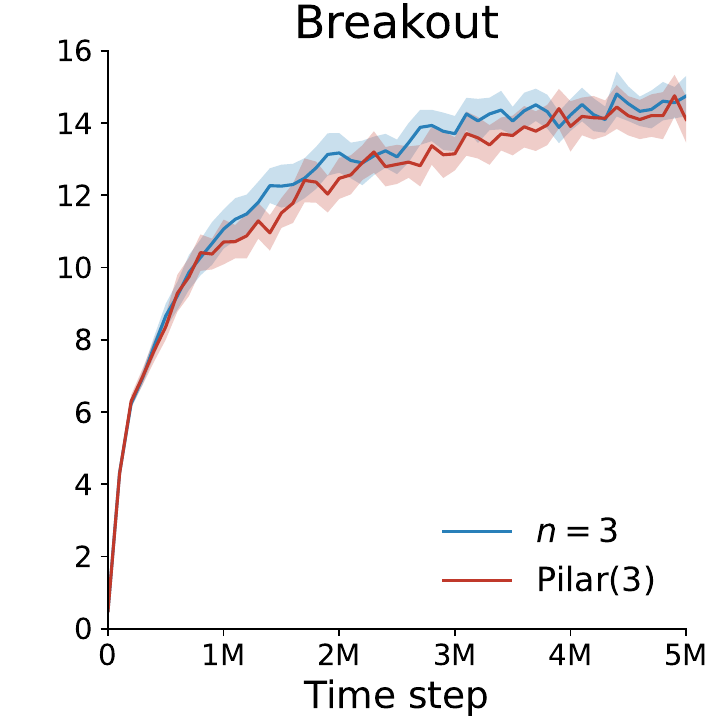}
    \hfill
    \includegraphics[width=0.24\textwidth]{dqn_breakout_n5_no_ylabel.pdf}

    \vspace{0.15in}

    \includegraphics[width=0.24\textwidth]{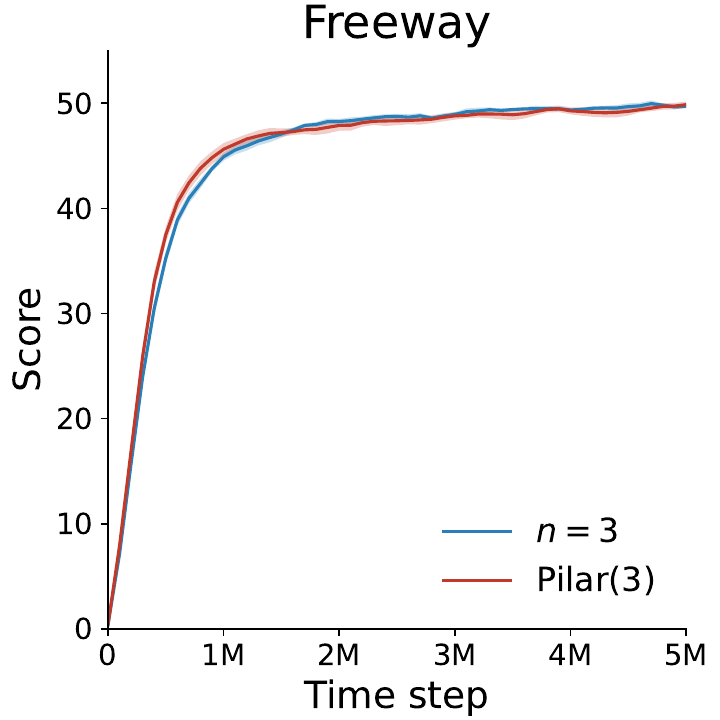}
    \hfill
    \includegraphics[width=0.24\textwidth]{dqn_freeway_n5_no_ylabel.pdf}
    \hfill
    \includegraphics[width=0.24\textwidth]{dqn_seaquest_n3_no_ylabel.pdf}
    \hfill
    \includegraphics[width=0.24\textwidth]{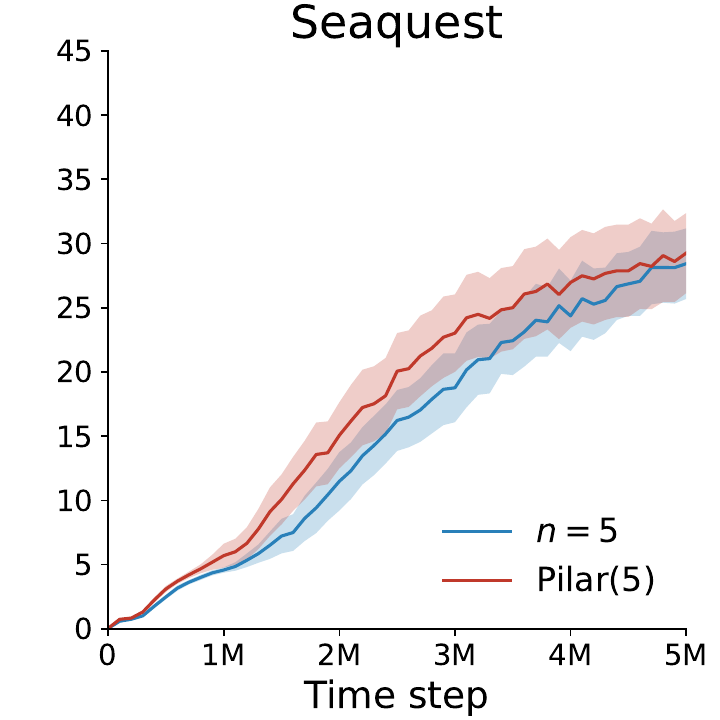}

    \vspace{0.15in}

    \hfill
    \includegraphics[width=0.24\textwidth]{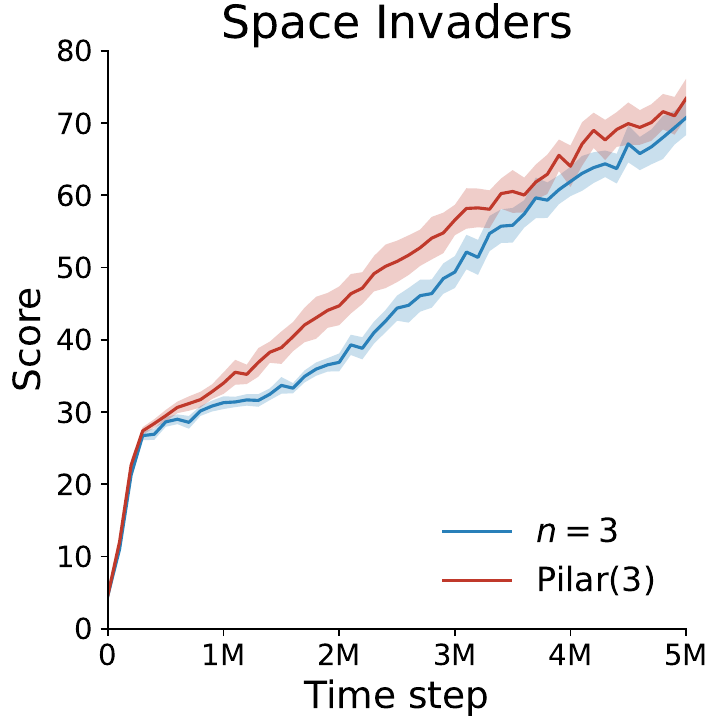}
    \hfill
    \includegraphics[width=0.24\textwidth]{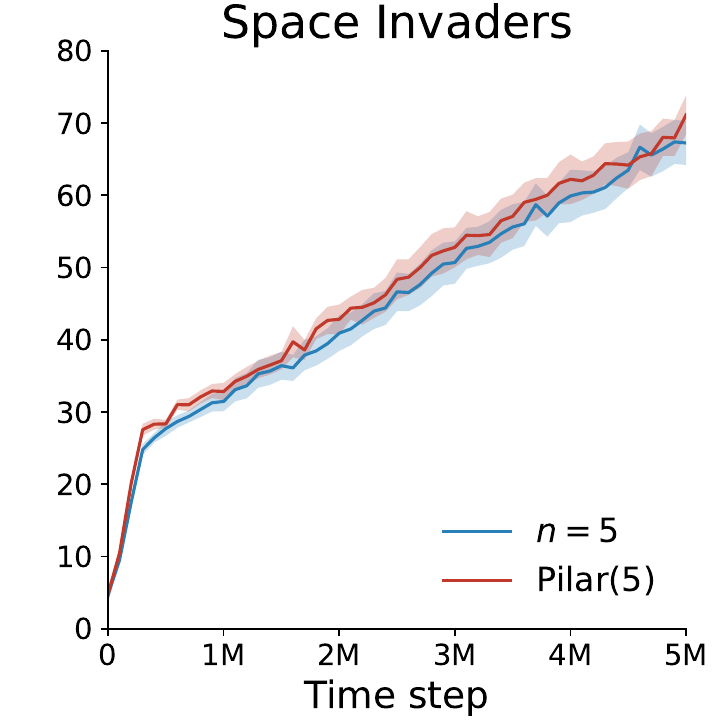}
    \hfill
    \caption{
        Learning curves for DQN with $n$-step returns and Pilars in five MinAtar games.
    }
    \label{fig:lcurves_dqn_all}
\end{figure}

\begin{figure}[ht]
    \centering
    \includegraphics[width=0.32\textwidth]{ppo_halfcheetah_n5.pdf}
    \hfill
    \includegraphics[width=0.32\textwidth]{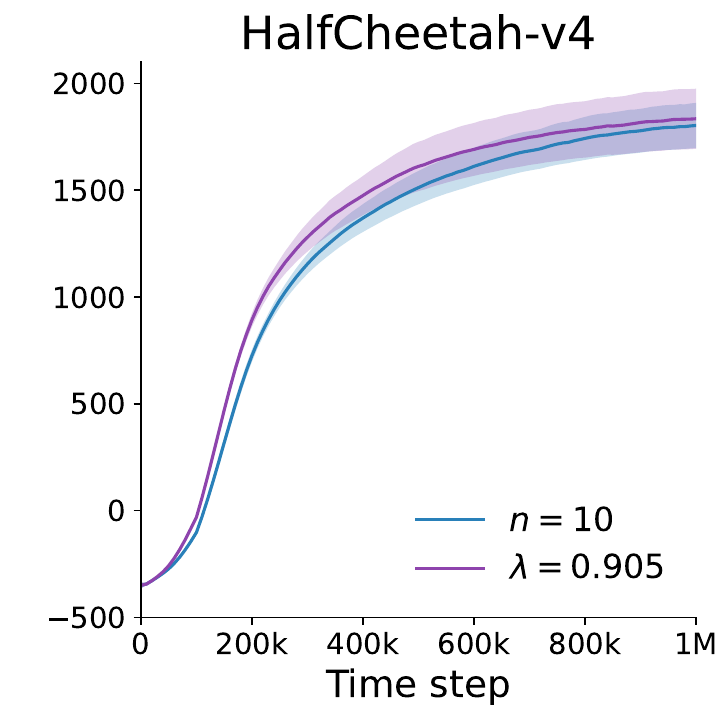}
    \hfill
    \includegraphics[width=0.32\textwidth]{ppo_halfcheetah_n20_no_ylabel.pdf}

    \vspace{0.15in}

    \includegraphics[width=0.32\textwidth]{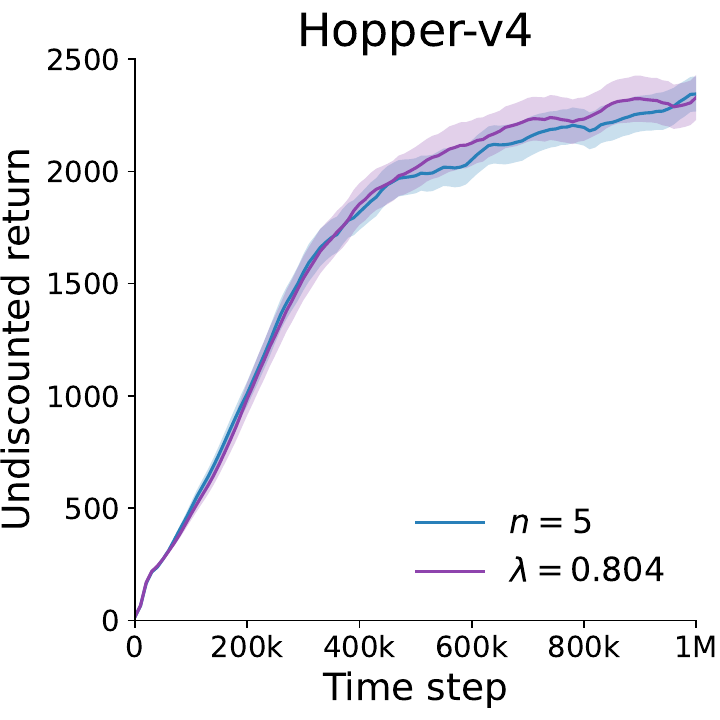}
    \hfill
    \includegraphics[width=0.32\textwidth]{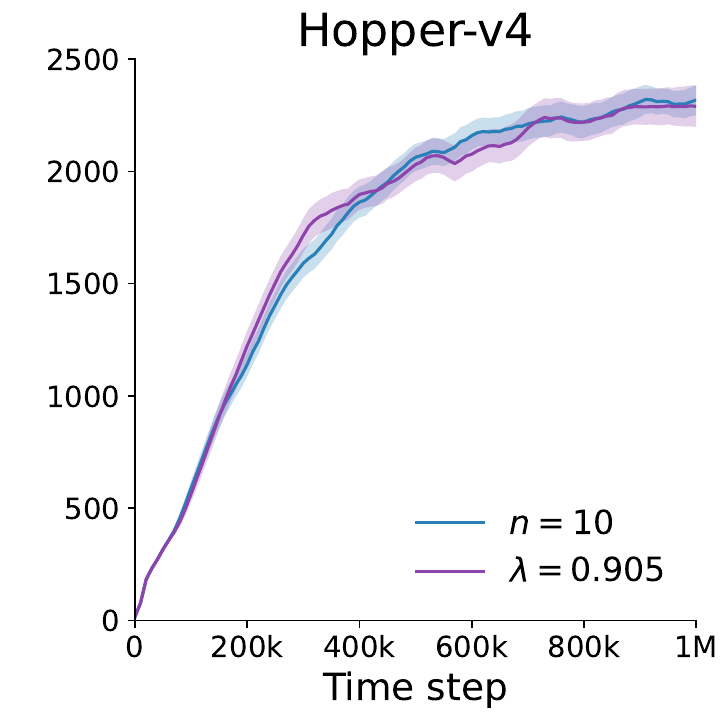}
    \hfill
    \includegraphics[width=0.32\textwidth]{ppo_hopper_n20_no_ylabel.pdf}

    \vspace{0.15in}

    \includegraphics[width=0.32\textwidth]{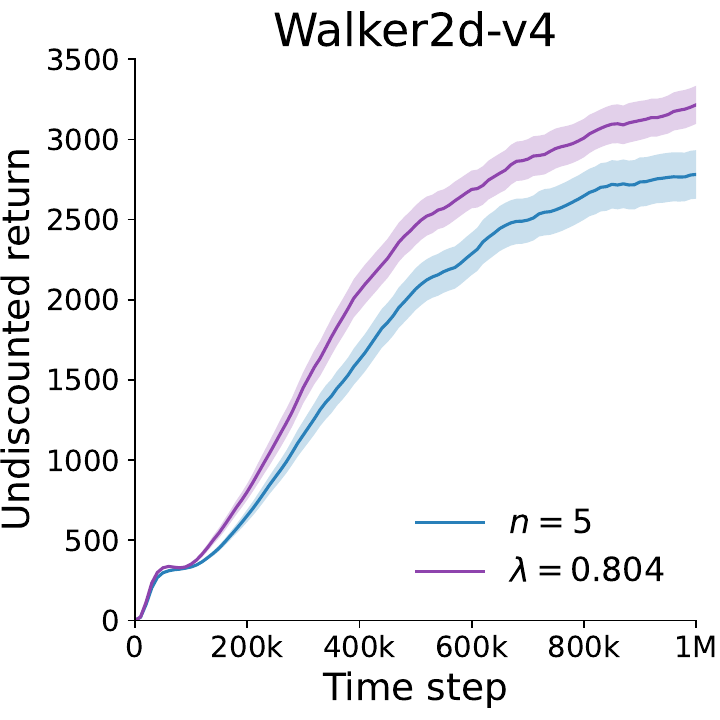}
    \hfill
    \includegraphics[width=0.32\textwidth]{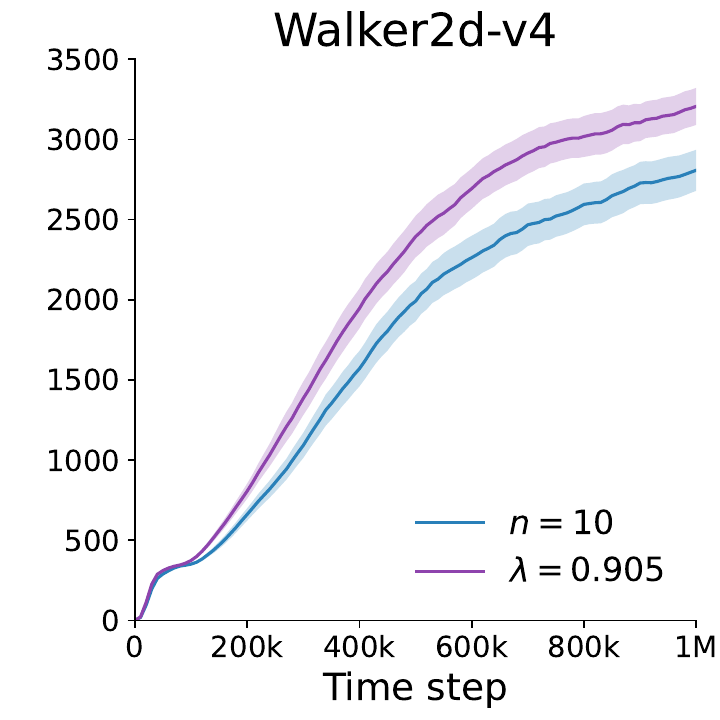}
    \hfill
    \includegraphics[width=0.32\textwidth]{ppo_walker2d_n20_no_ylabel.pdf}
    \caption{
        Learning curves for PPO with $n$-step returns and $\lambda$-returns in three MuJoCo environments.
    }
    \label{fig:lcurves_ppo_all}
\end{figure}

%%%%%%%%%%%%%%%%%%%%%%%%%%%%%%%%%%%%%%%%%%%%%%%%%%%%%%%%%%%%%%%%%%%%%%%%%%%%%%%
%%%%%%%%%%%%%%%%%%%%%%%%%%%%%%%%%%%%%%%%%%%%%%%%%%%%%%%%%%%%%%%%%%%%%%%%%%%%%%%

\end{document}